\documentclass[16pt, a4paper]{article}
\usepackage{amsmath,amsfonts,amsthm,mathrsfs}
\usepackage[page,title,titletoc,header]{appendix}
\usepackage{color}
\usepackage{algorithm}
\usepackage{algorithmic}
\usepackage{graphicx}
\usepackage{subfigure}

\newtheorem{assumption}{\textbf{Assumption}}
\newtheorem{theorem}{\textbf{Theorem}}

\newtheorem{lemma}{\textbf{Lemma}}

\newtheorem{definition}{\textbf{Definition}}

\allowdisplaybreaks
\begin{document}

\title{Analysis of Regularized Federated Learning}


\author{Langming Liu \\
School of Data Science, City University of Hong Kong \\
Kowloon, Hong Kong \\
Email: langmiliu2-c@my.cityu.edu.hk \\  \\
Ding-Xuan Zhou \\
School of Mathematics and Statistics, University of Sydney \\
 Sydney NSW 2006, Australia \\
Email: dingxuan.zhou@sydney.edu.au}
\date{}

\maketitle


\begin{abstract}
Federated learning is an efficient machine learning tool for dealing with heterogeneous big data and privacy protection. Federated learning methods with regularization can control the level of communications between the central and local machines. Stochastic gradient descent 
is often used for implementing such methods on heterogeneous big data, to reduce the communication costs. In this paper, we consider such an algorithm called Loopless Local Gradient Descent which has advantages in reducing the expected communications by controlling a probability level. We improve the method by allowing flexible step sizes and carry out novel analysis for the convergence of the algorithm in a non-convex setting in addition to the standard strongly convex setting. In the non-convex setting, we derive rates of convergence when the smooth objective function satisfies a Polyak-Łojasiewicz condition. When the objective function is strongly convex, a sufficient and necessary condition for the convergence in expectation is presented.
\end{abstract}







\section{Introduction}

Nowadays, we face more and more big data that cannot be processed by traditional machine learning (ML) methods due to the size or collected to a single machine due to the issue of privacy protection. A well-developed approach to tackle this, distributed learning~\cite{dean2012large, teerapittayanon2017distributed,dede1996evolution}, is to process data subsets by local machines separately, send the individual model parameter updates to a central machine for global model updates, and then communicate back to the local machines. This approach needs high levels of data communication and may have difficulty in handling heterogeneous data. 
Federated learning (FL)  \cite{mcmahan2016federated, konevcny2016federated, bonawitz2019towards, li2020review, yang2019federated, kairouz2021advances} 
was proposed to overcome these difficulties. FL is a popular AI basic technology designed for learning involving many independent local machines such as mobile phones of individual users which can process limited scales of data effectively. It has some advantages in data privacy protection, data security, data access rights and access to heterogeneous data, and has powerful applications in financial technologies, telecommunications, IoT, defence, pharmaceutics, and some other practical domains. 

The prevalent optimization formulation of FL~\cite{li2020federated} is an empirical risk minimization problem 
\begin{equation}
\min\limits_{x\in \mathbb{R}^d}\frac{1}{n}\sum_{i=1}^n f_i(x),
\label{FL}
\end{equation}
where $n\in\mathbb{N}$ is the number of local machines or devices, $x\in \mathbb{R}^d$ encodes the $d$ parameters of a global model (e.g., weights and biases of deep neural networks), and
$f_i(x)$ represents the aggregate loss of the $i$-th local machine when the model parameter $x$ is used. The loss $f_i(x)$ may have the form $f_i(x)=E_{\xi\sim D_i}[f(x,\xi)]$ when the local error induced by a sample $\xi$ and a model parameter $x$ equals $f(x,\xi)$ and $D_i$ is the data distributed on the $i$-th local machine. Note that $D_i$ can be remarkably different across local machines, representing the data heterogeneity caused by the diversity in interests, financial markets, occupations, ages, genders, and many other factors. Ideal federated learning models should perform well for dealing with heterogeneous data. 

A classical approach to solve~\eqref{FL} is the FederatedAveraging (FedAvg) algorithm~\cite{mcmahan2016federated}. This method performs well in dealing with imbalanced data distributions and non-convex problems and allows high-quality models to be trained in relatively few rounds of communications. However, FedAvg comes with poor convergence guarantees when data are heterogeneous. It also fails to provide theoretical improvements in the communication complexity over gradient descent algorithms, even its variants such as local gradient descent~\cite{khaled2019first}. 

In this paper we are interested in a unified formulation of FL models introduced in~\cite{hanzely2020federated} which is a regularized empirical risk minimization problem 
\begin{equation}
\min\limits_{x_1,\cdots,x_n\in \mathbb{R}^d} \{F(x):= f(x)+\lambda \psi(x)\}, 
\label{RFL}
\end{equation}
where $0\leq \lambda \leq \infty$ is a regularization or penalty parameter, $x:=(x_i)_{i=1}^n \in \mathbb{R}^{nd}$ is the model vector formed by the model parameters $x_1,\cdots, x_n\in \mathbb{R}^d$ for the $n$ local machines, 
$$f(x):=\frac{1}{n}\sum_{i=1}^n f_i(x_i) $$
is the global error function associated with the individual aggregate loss $\{f_i: \mathbb{R}^n \to \mathbb{R}\}_{i=1}^n$, and 
$$\psi(x):=\frac{1}{2n}\sum_{i=1}^n \Vert x_i-\bar{x}\Vert^2$$
is the regularization term with $\bar{x}=\frac{1}{n}\sum_{i=1}^n x_i$ being the average of the local model parameters. 

The regularization parameter $\lambda$ for the regularization term measures the level of communications in the FL system where the central and local machines exchange information such as gradients or model parameters. 
When $\lambda=0$, the optimization problem (\ref{RFL}) becomes a local one, meaning that we solve the optimization problem by minimizing each $f_i$ over $\mathbb{R}^d$ separately. In this situation, we do not need communications. When $\lambda=\infty$, each $x_i$ equals to the same vector $\bar{x}$ and (\ref{RFL}) is equivalent to the global optimization problem~\eqref{FL} for $\bar{x}\in \mathbb{R}^d$.
When $0<\lambda <\infty$, we consider a mixed optimization problem, where the regularization term controls how the models in the individual local machines are similar. In other words, the first term of~\eqref{RFL} involves the $n$ individual local aggregate losses. In contrast, the regularization term guarantees the individual local models to be close to the average one at a certain level. An appropriate level of communications are needed for solving the mixed problem in this case. 

When the size $nd$ of the data is large, a natural method for solving the optimization problem (\ref{RFL}) is the standard gradient descent (GD), involving updates of the full gradient of $F$ in each iteration. But the communication cost would be extremely high with this method.  To see this, notice that for each iteration, the gradients for the losses of the local machines are calculated and sent to the central machine, then these are averaged and returned to the local machines (locals$\rightarrow$central$\rightarrow$locals). Therefore GD needs two rounds of communications per iteration, which is impractical for many real-world implementations with big data such as mobile updating. To reduce the communication costs of GD, a stochastic gradient descent (SGD) method called L2GD (Loopless Local Gradient Descent) was introduced in~\cite{hanzely2020federated}. 

The idea of this method is to separate the gradient of $F$ into those of $f$ and $\psi$ with probability, controlled by a probability level $0<p<1$. Define a stochastic gradient of $F$ at $x\in \mathbb{R}^{nd}$ as 
\begin{equation}
\label{SG}
G(x):=\left\{
\begin{aligned}
&\frac{\nabla f(x)}{1-p}\quad&\text{with probability }&\quad&1-p,\\
&\frac{\lambda\nabla\psi(x)}{p}&\text{with probability }& &p.\quad
\end{aligned}\right.
\end{equation}
It is straightforward that $G(x)$ is an unbiased estimator of $\nabla F(x)$, which means $\mathbb{E}[G(x)]=\nabla F(x)$. This unbiased estimate for the gradient leads to the updating rule of L2GD as 
\begin{equation}
x^{k+1}=x^k-\alpha G(x^k),
\label{L2GD}
\end{equation}
where $0<\alpha<\infty$ is the step size (or learning rate) of the algorithm.  
Specifically, in each iteration, we generate a Bernoulli random variable $\xi$ with $\mathbf{Pr}(\xi=1)=p$ and $\mathbf{Pr}(\xi=0)=1-p$. If $\xi=1$, L2GD will 
calculate the average model $\bar{x}$ and force the local model $x_i$ closing to the average. Otherwise, L2GD will process the local GD step for all the local machines. Communications at an iteration are needed only if $\xi_k=0, \xi_{k+1}=1$. An advantage of L2GD is to reduce the expected communications by controlling the probability $p$. 


Convergence of L2GD was considered in \cite{hanzely2020federated} when the step size is fixed. Let $x(\lambda)$ be the minimizer of~\eqref{RFL}. 
When each function $f_i$ with $i\in\{1, \ldots, n\}$ is $L-$smooth and $\mu$-strongly convex (to be defined below) and $0<\alpha\le\frac{1}{2\mathcal{L}}$, it was shown in~\cite{hanzely2020federated} that 
\begin{equation}
\label{convegence_rate}
   \mathbb{E}[\Vert x^k-x(\lambda)\Vert^2]\le(1-\frac{\alpha\mu}{n})^k\Vert x^0-x(\lambda)\Vert^2+\frac{18n\alpha\sigma^2}{\mu}, \qquad \forall k\in\mathbb{N}, 
\end{equation}
where $\mathcal{L}:=\frac{1}{n}
\max\{\frac{(1+2p)L}{1-p},\frac{(3-2p)\lambda}{p}\}$, and
\[\sigma^2:=\frac{1}{n^2}\sum_{i=1}^{n}\left(\frac{1}{1-p}\Vert\nabla f_i(x_i(\lambda))\Vert^2+\frac{\lambda^2}{p}\Vert x_i(\lambda)-\bar{x}(\lambda)\Vert^2\right).\]

The first term of the estimate stated in (\ref{convegence_rate}) decays to zero linearly, but the second term involves a variance term $\sigma^2$, which is a constant and does converge to $0$ in general. Thus, the algorithm given by L2GD using a fixed step size is sub-optimal in terms of convergence.  

There have been several variance-reduction methods~\cite{gower2020variance} to overcome the above-mentioned difficulty. One is stochastic average gradient (SAG)~\cite{roux2012stochastic}. The idea behind SAG is to use an estimate $v_k^i\approx\nabla f_i(x_k)$ for each $i$ and then to compute the full gradient by using the average of the $v_k^i$ values. The second is stochastic dual coordinate ascent (SDCA)~\cite{shalev2013stochastic}, based on the property that the coordinates of the gradient provide a naturally variance-reduced estimate of the gradient. Stochastic variance-reduced gradient (SVRG) is another method proposed in~\cite{johnson2013accelerating} to further improve convergence rates over SAG. In addition, SVRG is the first work to use covariates to address the high memory of SAG. This was followed by SAGA~\cite{defazio2014saga}, a variant of SAG applying the covariates to make an unbiased variant of SAG that has similar performances but is easier to analyze. Then, a class of algorithms, named uniform memorization algorithms~\cite{hofmann2015variance}, are proposed inspired by SVRG and SAGA. Variance-reduced Stochastic Newton (VITE)~\cite{ying2020variance} utilizes Newton's methods, basically a first-order approximation, to obtain fast convergence rates. Then, a new amortized variance-reduced gradient (AVRG) algorithm is proposed, which possesses storage and computational efficiency comparable to SAGA and SVRG. Recently, a stochastic quasi-gradient method, named JacSketch~\cite{gower2021stochastic}, leverages the approximation of the Jacobian matrix to update models. 

One may apply the above variance-reduction methods to the FL optimization problem~\eqref{RFL} to remove the variance term in (\ref{convegence_rate}). However, new issues emerge with more computational and memory costs and extra communication costs under the FL context, a trade-off between convergence and efficiency. For example, since some VR methods~\cite{johnson2013accelerating,defazio2014saga,hofmann2015variance, lucchi2015variance, yuan2018variance} introduce auxiliary variates, 
we need to compute and store them each time, and then communicate them between local machines and the server. 

To address the aforementioned issues, we propose \textbf{L}oopless \textbf{L}ocal \textbf{G}radient \textbf{D}escent utilizing a \textbf{V}arying step size, named L2GDV. We will utilize a varying step size instead of a fixed one. This approach is intuitive and reasonable since, for the training process, when the trained model gets closer and closer to the best one, we should slow down the step to guarantee that after enough iterations, the model will be stable around the minimum. One can easily see similar computational, memory, and communication efficiencies as the basic SGD methods like L2GD. The scientific question is to determine rigorously if a varying step size suffices to guarantee convergence (without a variance term) and how fast the convergence is. This is answered by our theoretical and experimental results in Section~\ref{theo_results} and Section~\ref{exp_results} with the following  contributions:

\begin{itemize}
\item We extend the convergence results for algorithm (\ref{L2GD}) to a non-convex situation, where each function $f_i$ is $L-$smooth, $F$ satisfies a PL condition (to be defined below), and a step size is fixed. 

\item For the non-convex situation, we provide rates of convergence when the step size sequence decays polynomially as $\alpha_k = \alpha_1 k^{-\theta}$ with $0<\theta \leq 1$ and $\alpha_1>0$. 

\item We show in a convex situation where each function $f_i$ is $L-$smooth and $\mu$-strongly convex with $L, \mu >0$ and $0<\alpha_k\le\frac{1}{2\mathcal{L}}$, the algorithm (\ref{L2GDnew}) converges with $\lim_{k\to\infty}\mathbb{E}[\Vert x^k-x(\lambda)\Vert^2]=0$ if and only if $\lim_{k\to\infty}\alpha_k=0$ and $\sum_{k=1}^\infty\alpha_k=\infty$. 
    
\item For the convex situation, we provide rates of convergence when the step size sequence decays polynomially as $\alpha_k = \alpha_1 k^{-\theta}$ with $0<\theta \leq 1$ and $0<\alpha_1 \le\frac{1}{2\mathcal{L}}$. 

\item We conduct experiments in the non-convex and convex situations to demonstrate the effectiveness of the proposed method, compared with strong baselines in federated learning. 
\end{itemize}

\section{Methodology}

\subsection{Problem statement}

We state the problem formulation here again for convenience. The optimization objective in this paper is a regularized ERM problem~\cite{hanzely2020federated} 
\begin{equation}
\min\limits_{x_1,\cdots,x_n\in \mathbb{R}^d} \{F(x):= f(x)+\lambda \psi(x)\}, 
\label{RFL_2}
\end{equation}
$$f(x):=\frac{1}{n}\sum_{i=1}^n f_i(x_i),\quad \psi(x):=\frac{1}{2n}\sum_{i=1}^n \Vert x_i-\bar{x}\Vert^2$$
with a global error function $f$, regularization term $\psi$, regularization parameter $0\leq \lambda \leq \infty$, and the model vector $x:=(x_i)_{i=1}^n \in \mathbb{R}^{nd}$ together with the average $\bar{x}=\frac{1}{n}\sum_{i=1}^n x_i$ of the local model parameters.

The parameter $\lambda$ controls the relative importance of global error and regularization term. Two extreme cases occur when $\lambda=0$ and $\lambda=\infty$. The former means we solve $n$ local minimization problems without the need of communication. The latter means the regularization term forces all local model $x_i$ equal to the average model $\bar{x}$, which is the same as global problem~\eqref{FL}. We consider local and global optimization problems jointly by letting $0<\lambda <\infty$, where the first term encourages local models to update individually and the second term guarantees local models close to each other.

\subsection{Algorithm: L2GDV}

We propose L2GDV to solve the problem~\eqref{RFL_2}. 
L2GDV leverages a varying step size in the training process. 
Specifically, our updating rule takes the form 
\begin{equation}\label{L2GDnew}
x^{k+1}=x^k-\alpha_k G(x^k),
\end{equation}
where $\{\alpha_k\}_{k\in\mathbb{N}}$ is the step size sequence of the algorithm. In this paper, we consider a general decaying step size sequence with $\alpha_k = \alpha_1 k^{-\theta}$, and $G(x)$ is a non-uniform stochastic gradient of $F$ at $x\in \mathbb{R}^{nd}$ given by (\ref{SG}). 

Note that L2GD is a special case of (\ref{L2GDnew}) with a constant step size sequence $\alpha_k \equiv \alpha$. Since the functions $f$ and $\psi$ are both separable for each model $x_i$ with 
$$ \nabla\psi(x)=\frac{1}{n}\left(x_i-\bar{x}\right)_{i=1}^n, $$
the updating rule for the local machine can be formulated as   
\begin{equation}
x_i^{k+1}=x_i^k-\alpha_k G_i(x_i^k),
\end{equation}
where stochastic gradient $G_i$ has a similar form as $G$ by substituting global losses $f,\psi$ to local losses $f_i,\psi_i$ with $\psi_i(x_i^k)=\frac{1}{n}\left(x_i^k-\overline{x^k}\right)$. We provide a pseudo-code in Algorithm~\ref{algo:L2GD} to demonstrate our updating rule. Compared to FedProx~\cite{li2020federated1}, the updating rule here utilizes a non-uniform SGD method, which randomly conducts Local GD and global regularization (moving local model to average), enhancing the communication efficiency~\cite{hanzely2020federated}. Then we will give the theoretical and experimental results to show the effectiveness of proposed method.

\begin{algorithm}
	\caption{L2GDV (using a varying step size)}
	\label{algo:L2GD}
	\begin{algorithmic}[1]
		\STATE\textbf{Input:} $n$ local machines and server
		\STATE\textbf{Initialization:} models $x^1 = \{x_i^1\}_{i=1}^n$, parameters $\alpha_1,p,\theta$
        \FOR{$k=1,2,\cdots, K$} 
            \STATE generates random variable $\xi_k = 1$ with prob $p$ and $0$ with prob $1-p$
            \STATE generates varying step size $\alpha_k = \alpha_1 k^{-\theta}$
            \IF{$\xi_k = 0$}
                \STATE local client update
                \FOR{$i=1,\cdots, n$ in parallel} 
                    \STATE \textbf{main update}: $x_i^{k+1} = x_i^{k}-\frac{\alpha_k}{n(1-p)}\nabla f_i(x_i^{k})$
                \ENDFOR
            \ELSE
                \STATE central server aggregate: $\overline{x^{k}} = \frac{1}{n}\sum_{i=1}^n {x}_{i}^k$
                \STATE \textbf{moving to average}: $x_{i}^{k+1} = x_{i}^{k}-\frac{\alpha_k}{np}\lambda(x_{i}^{k}-\overline{x^{k}})$
            \ENDIF
        \ENDFOR
	\end{algorithmic}  
\end{algorithm}

\section{Theoretical Results}
\label{theo_results}
Considering that the individual aggregate losses $\{f_i\}$ are usually non-convex, we do not assume the strong convexity as in \cite{hanzely2020federated}. Our main results, to be proved in Appendix~\ref{proof_nonconvex}, are stated for a general setting~\cite{karimi2016linear, Davis2019, Zeng2018, ZengYinZhou2022} for solving non-convex ML problems, where a PL condition holds for $F$ and losses $\{f_i\}$ are smooth. Then, we state finer results in a special case where a strong convexity is satisfied for $\{f_i\}$, which is proved in Appendix~\ref{proof_convex}. 

We say a function $g:\mathbb{R}^d\to \mathbb{R}$ satisfies a Polyak-Łojasiewicz (PL) condition with $\mu>0$ if $g$ achieves a minimum function value $g^*$, and for each $x$, 
\begin{equation}\label{PLdef}
    \frac{1}{2}\Vert \nabla g(x)\Vert^2\ge \mu(g(x)-g^*). 
\end{equation}
We say a function $g:\mathbb{R}^d\to \mathbb{R}$ is
$L-$smooth with $L>0$ if
\begin{equation}\label{smoothdef}
g(x)\le g(y)+\nabla g(y)^T(x-y)+\frac{L}{2}\Vert x-y\Vert^2, \quad\forall x,y\in \mathbb{R}^d,
\end{equation}
it is $\mu$-strongly convex with $\mu >0$ if
\begin{equation}\label{strongconvexdef}
g(x)\ge g(y)+\nabla g(y)^T(x-y)+\frac{\mu}{2}\Vert x-y\Vert^2 ,\quad\forall x,y\in \mathbb{R}^d.
\end{equation}

\subsection{Convergence results in a non-convex situation}
\label{sec:nonconvex}
\begin{assumption}
\label{assumption_0}
The objective function $F$ satisfies a PL condition with some $\mu>0$ and each $f_i: \mathbb{R}^d\to \mathbb{R}$ with $i\in \{1,2,\cdots, n\}$ is $L$-smooth for some $L>0$.
\end{assumption}

In the non-convex situation, since a function $g$ may have multiple minimizers, we denote $\mathcal{X}^*=\mathcal{X}^*_g= \{x\vert g(x)=g^*\}$ as the solution set of minimizers of $g$. Below we give rates of convergence when we fix the step size. 

\begin{theorem}
\label{theorem_PL_fixed}
Under Assumption~\ref{assumption_0}, when $0<\alpha\le\frac{\mu^2}{2\zeta L_F}$, the sequence $\{x^k\}$ defined by (\ref{L2GDnew}) with $\alpha_k \equiv \alpha$ satisfies 
\begin{equation}
\label{convegence_rate_nonconvex}
   \mathbb{E}[F(x^{k+1})-F^*]\le(1-\mu\alpha)^k[F(x^1)-F^*]+\frac{9\alpha L_F\sigma_{m}^2}{\mu}, \qquad \forall k\in\mathbb{N}, 
\end{equation}
where $L_F :=\frac{L + \lambda}{n}, \sigma_{m}^2=\max\limits_{x \in\mathcal{X}^*}\{\sigma_{x}^2\}$, $\zeta = \frac{(1+2p)L^2}{(1-p)n^2}+\frac{(3-2p)\lambda^2}{pn^2}$, and $$\sigma_x^2:=\frac{1}{n^2}\sum_{i=1}^{n}\left(\frac{1}{1-p}\Vert\nabla f_i(x_{i})\Vert^2+\frac{\lambda^2}{p}\Vert x_{i}-\bar{x}\Vert^2\right).$$
\end{theorem}

The definition of $\sigma_m^2$ implies that the gradient variance of the points in the solution set is bounded, which is a weaker version of uniformly-bounded gradient variance condition~\cite{karimi2016linear,yuan2020federated,faw2022power,liu2020improved} in SGD analysis. 
The above theorem shows that with a variance term, L2GD converges to the optimal function value with a linear convergence rate under the non-convex setting. However, the variance term exists in the estimate similar to the results in the convex situation, which jeopardizes the convergence.  
Inspired by the previous works~\cite{ying2008online,lin2015learning}, we take a polynomially decaying step size sequence to eliminate the variance term under this non-convex situation, where $\alpha_k=\alpha_1 k^{-\theta}$. This design of a decaying step size sequence is sufficient to guarantee convergence, and the convergence rates are provided in the following theorem. 

\begin{theorem}
\label{theorem_PL_decay}
Under Assumption~\ref{assumption_0}, when $\alpha_k=\alpha_1 k^{-\theta}$ with $0<\alpha_1\le\frac{\mu^2}{2\zeta L_F}$ and $0<\theta \leq 1$, we have 
\begin{equation}\label{ratesofconvergence_PL}
\mathbb{E}[F(x^{k+1})-F^*] \le C_{\mu, \alpha_1, \theta, R^1} 
\left\{\begin{array}{ll} 
k^{-\theta}, & \hbox{if}  \ 0< \theta <1, \\
k^{-\mu \alpha_1}, & \hbox{if}  \ \theta =1, \mu \alpha_1 < 1, \\
\frac{1 + \log k}{k}, & \hbox{if}  \ \theta =1, \mu \alpha_1 = 1, \\
\frac{1}{k}, & \hbox{if}  \ \theta =1, \mu \alpha_1 > 1, 
\end{array}\right.
\end{equation} 
where $C_{\mu, \alpha_1, \theta, R^1, n}$ is a constant independent of $k$ given explicitly in the proof. 
\end{theorem}

\subsection{Convergence results in a convex situation}
\label{sec:convex}
\begin{assumption}
\label{assumption_1}
For some $L, \mu>0$, each $f_i: \mathbb{R}^d\to \mathbb{R}$ with $i\in \{1,2,\cdots, n\}$ is $L$-smooth and $\mu$-strongly convex.
\end{assumption}

The strong convexity in Assumption \ref{assumption_1} guarantees the existence of a unique solution $x(\lambda):=(x_i(\lambda))_{i=1}^n \in \mathbb{R}^{nd}$ of~\eqref{RFL_2}. Recall $\mathcal{L}:=\frac{1}{n}
\max\{\frac{(1+2p)L}{1-p},\frac{(3-2p)\lambda}{p}\}$.

\begin{theorem}
\label{theorem_1}
Under Assumption~\ref{assumption_1}, when $\sigma>0$, $x^1\not=x(\lambda)$, and $0<\alpha_k\le\frac{1}{2\mathcal{L}}$ for $k\in\mathbb{N}$, the sequence $\{x^k\}$ defined by (\ref{L2GDnew}) satisfies 
$\lim_{k\to\infty}\mathbb{E}_{x_k, x_{k-1}, \ldots, x_1}[\Vert x^{k+1}-x(\lambda)\Vert^2]=0$ if and only if
\begin{equation}
\lim_{k\to\infty}\alpha_k=0,\quad  \sum_{k=1}^\infty\alpha_k=\infty. 
\label{con}
\end{equation}
\end{theorem}

The conditions $\sigma>0$ and $x^1\not=x(\lambda)$ are needed only for the necessity in the theorem, as seen from the proof and some observations in the literature of online learning and Kaczmarz algorithms~\cite{ying2008online, lin2015learning}. In addition, the necessity holds only under the convex setting. 
Theorem \ref{theorem_1} tells us that for convergence, the step size sequence should converge to $0$, but it should not converge too fast. In particular, when we consider the decaying step size $\alpha_k=\alpha_1 k^{-\theta}$ with $\alpha_1, \theta>0$, we require $0<\theta\le1$. The following theorem provides rates of convergence for such step size sequences.

\begin{theorem}
\label{theorem_2}
Under Assumption~\ref{assumption_1}, when $\alpha_k=\alpha_1 k^{-\theta}$ with $0<\alpha_1\le\frac{1}{2\mathcal{L}}$ and $0<\theta \leq 1$, we have 
\begin{equation}\label{ratesofconvergence}
\mathbb{E}[\Vert x^{k+1}-x(\lambda)\Vert^2] \le C_{\mu, \alpha_1, \theta, \Vert r^{1}\Vert, n} 
\left\{\begin{array}{ll} 
k^{-\theta}, & \hbox{if}  \ 0< \theta <1, \\
k^{-\mu \alpha_1/n}, & \hbox{if}  \ \theta =1, \mu \alpha_1 <n, \\
\frac{1 + \log k}{k}, & \hbox{if}  \ \theta =1, \mu \alpha_1 =n, \\
\frac{1}{k}, & \hbox{if}  \ \theta =1, \mu \alpha_1 > n, 
\end{array}\right.
\end{equation} 
where $C_{\mu, \alpha_1, \theta, \Vert r^{1}\Vert, n}$ is a constant independent of $k$ given explicitly in the proof. 
\end{theorem}

\section{Experiments}
\label{exp_results}

\begin{figure}[t]
    \centering
    \subfigure[Test accuracy on MNIST IID.]{
        \begin{minipage}[t]{0.475\linewidth}
            \includegraphics[width=1.00\linewidth]{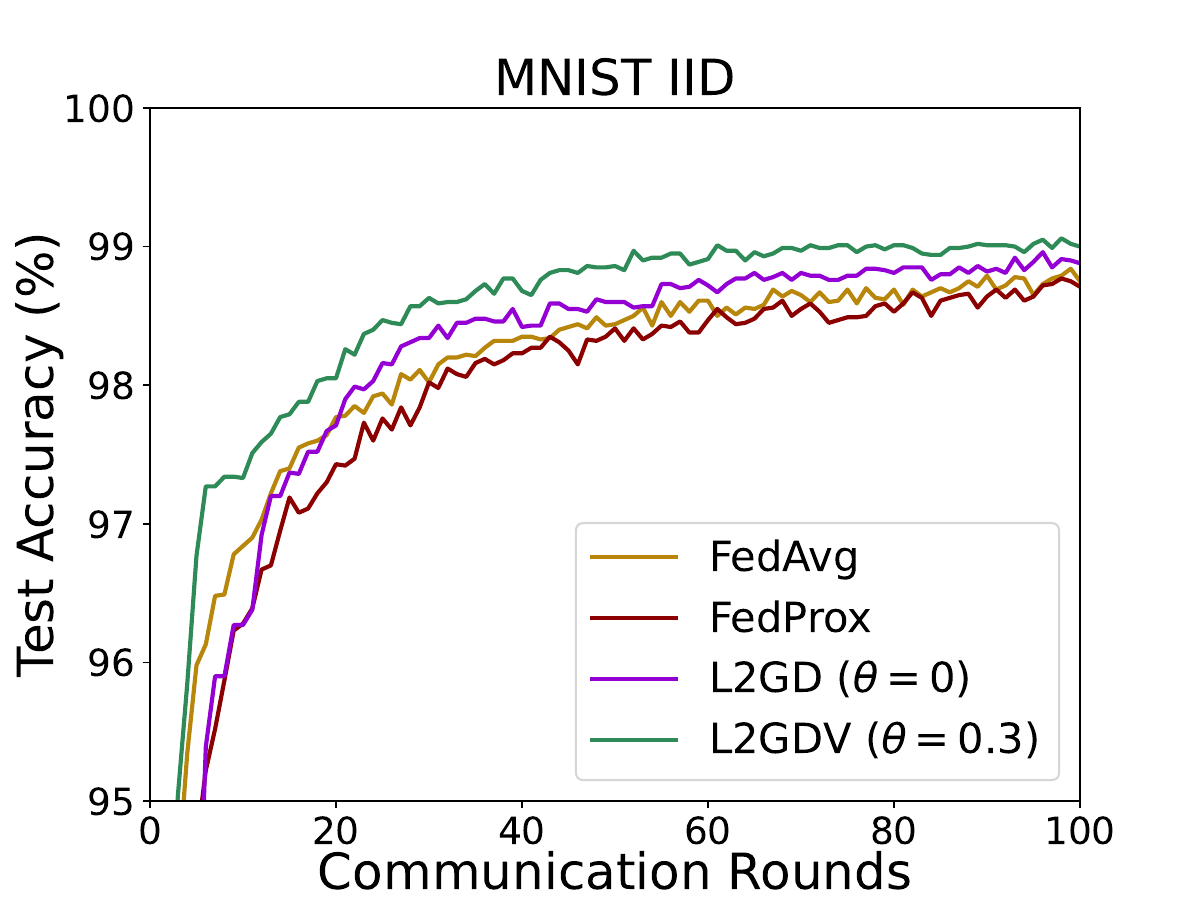}
        \label{subfig: ML-convergence}
        \end{minipage}
    }
    \subfigure[Train loss on MNIST IID.]{
        \begin{minipage}[t]{0.475\linewidth}
            \includegraphics[width=1.00\linewidth]{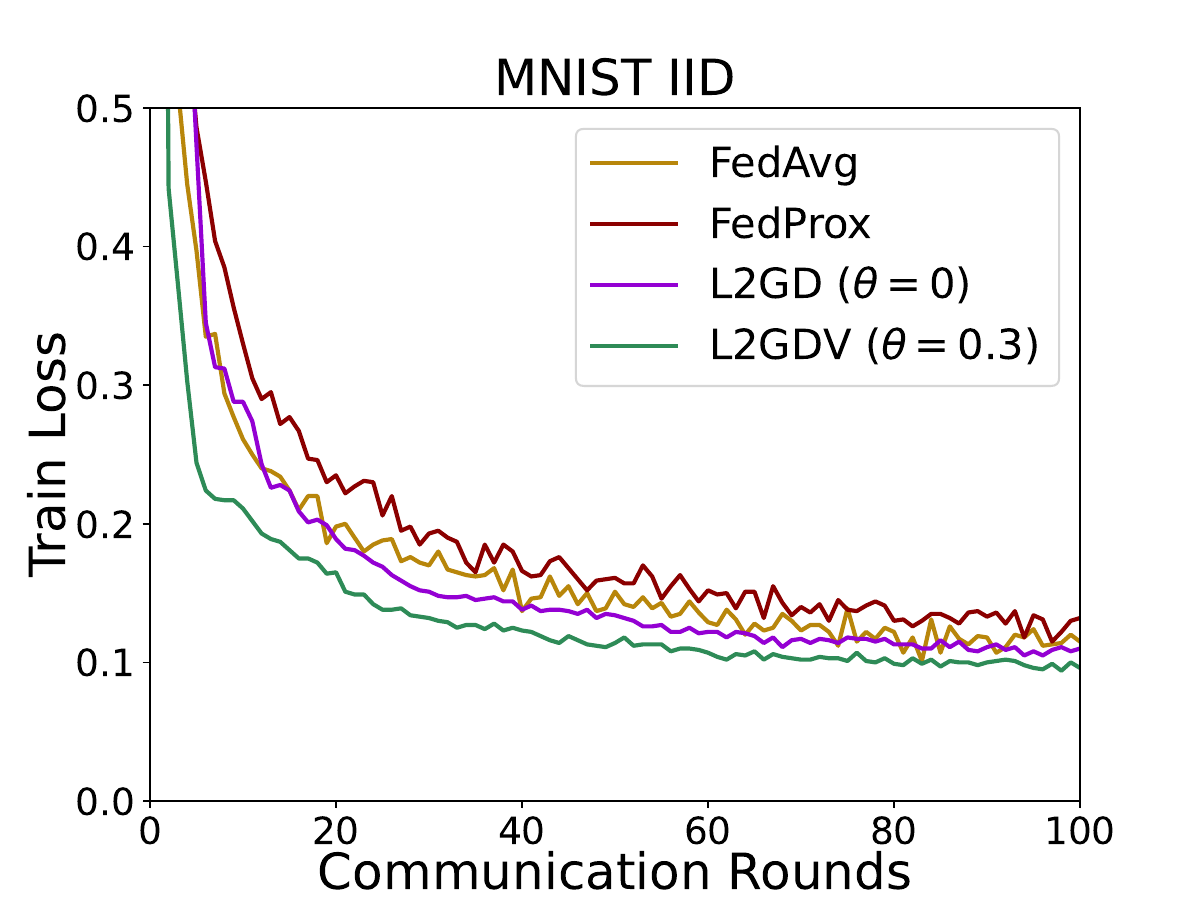}
        \label{subfig: ML-communication}
        \end{minipage}
    }
    \subfigure[Test accuracy on MNIST Non-IID.]{
        \begin{minipage}[t]{0.475\linewidth}
            \includegraphics[width=1.00\linewidth]{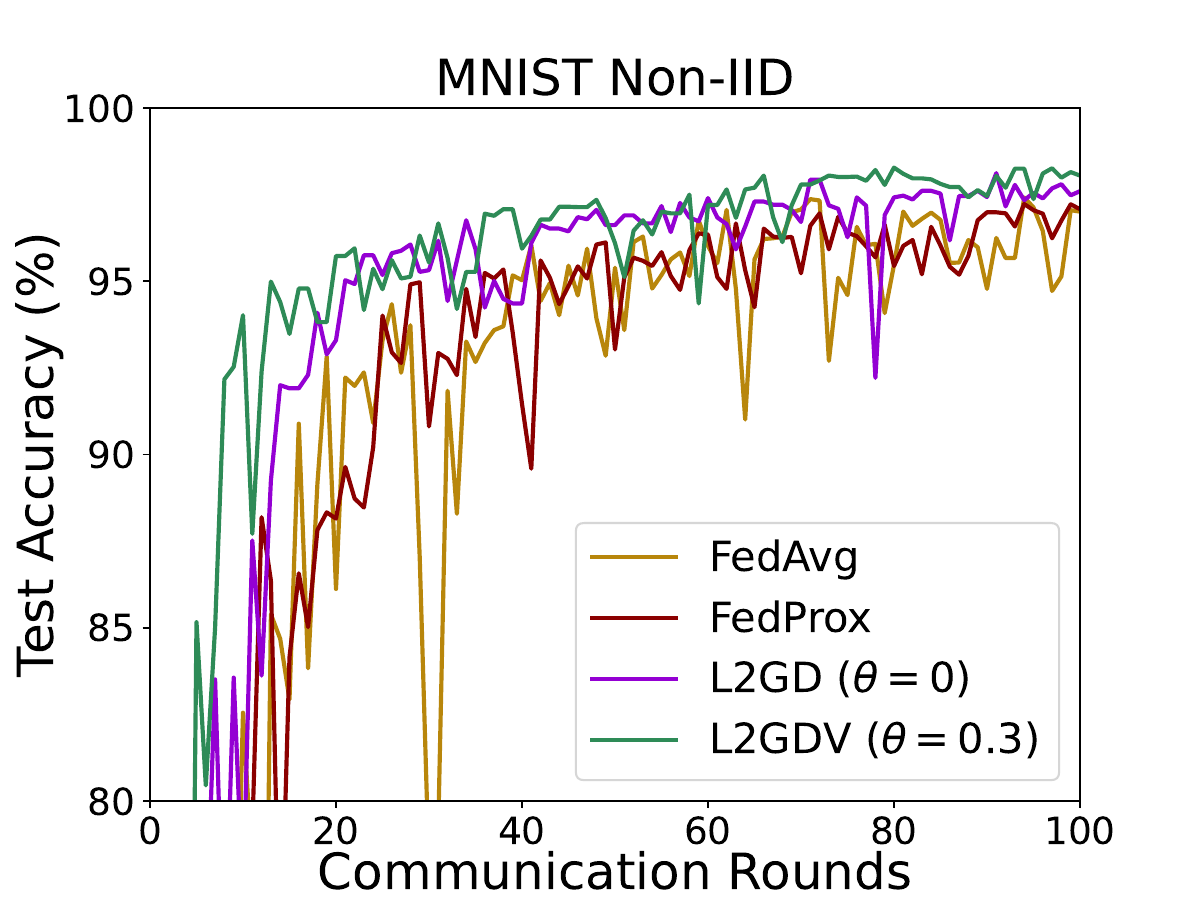}
        \label{subfig: Jester-convergence}
        \end{minipage}
    }
    \subfigure[Train loss on MNIST Non-IID.]{
        \begin{minipage}[t]{0.475\linewidth}
            \includegraphics[width=1.00\linewidth]{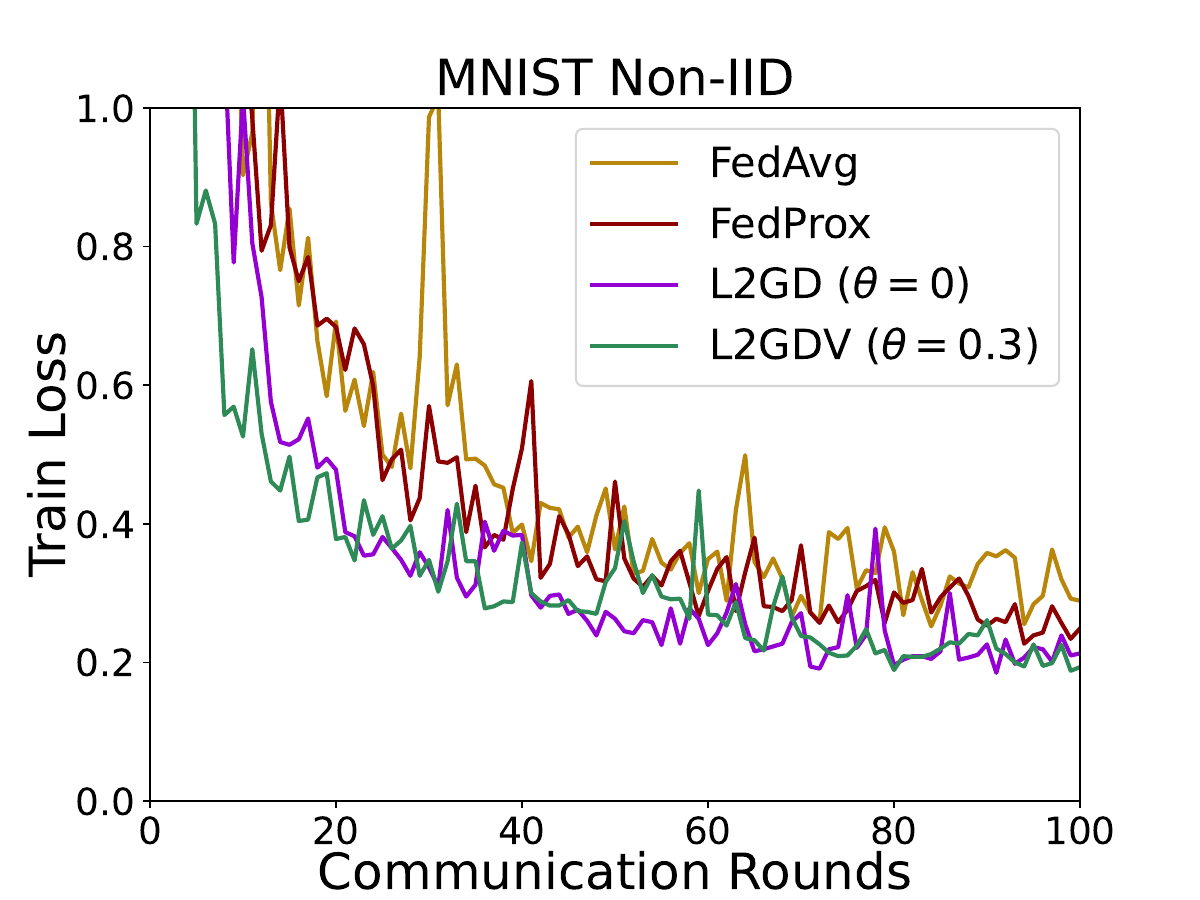}
        \label{subfig: Jester-communication}
        \end{minipage}
    }
    \caption{Test set accuracy and training set loss w.r.t. communication rounds in a non-convex situation (i.e., using the CNN model). Figures (a) and (b) show the results for MNIST IID, while Figures (c) and (d) show those for MNIST Non-IID.}
    \label{fig:nonconvex}
\end{figure}

\begin{figure}[t]
    \centering
    \subfigure[Test accuracy on MNIST IID.]{
        \begin{minipage}[t]{0.475\linewidth}
            \includegraphics[width=1.00\linewidth]{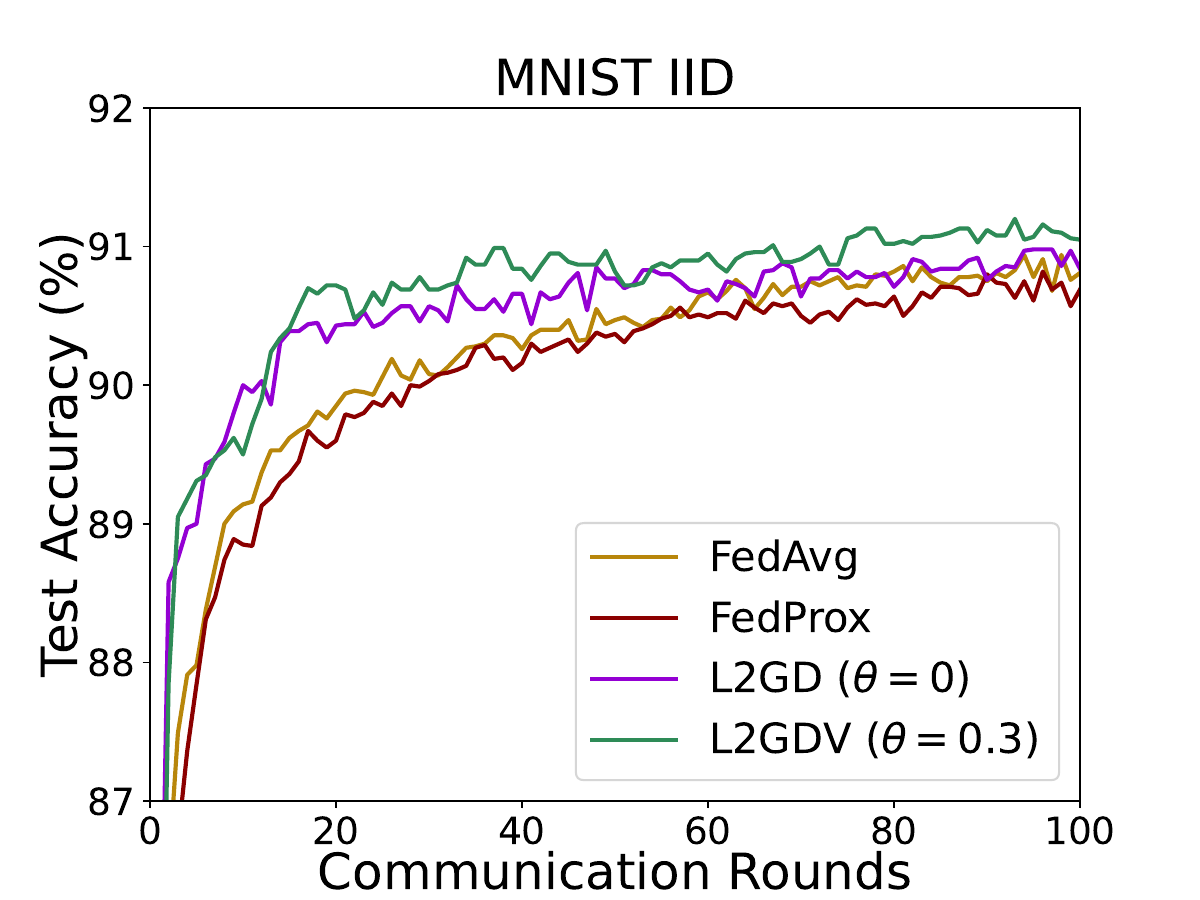}
        \label{subfig: ML-convergence}
        \end{minipage}
    }
    \subfigure[Train loss on MNIST IID.]{
        \begin{minipage}[t]{0.475\linewidth}
            \includegraphics[width=1.00\linewidth]{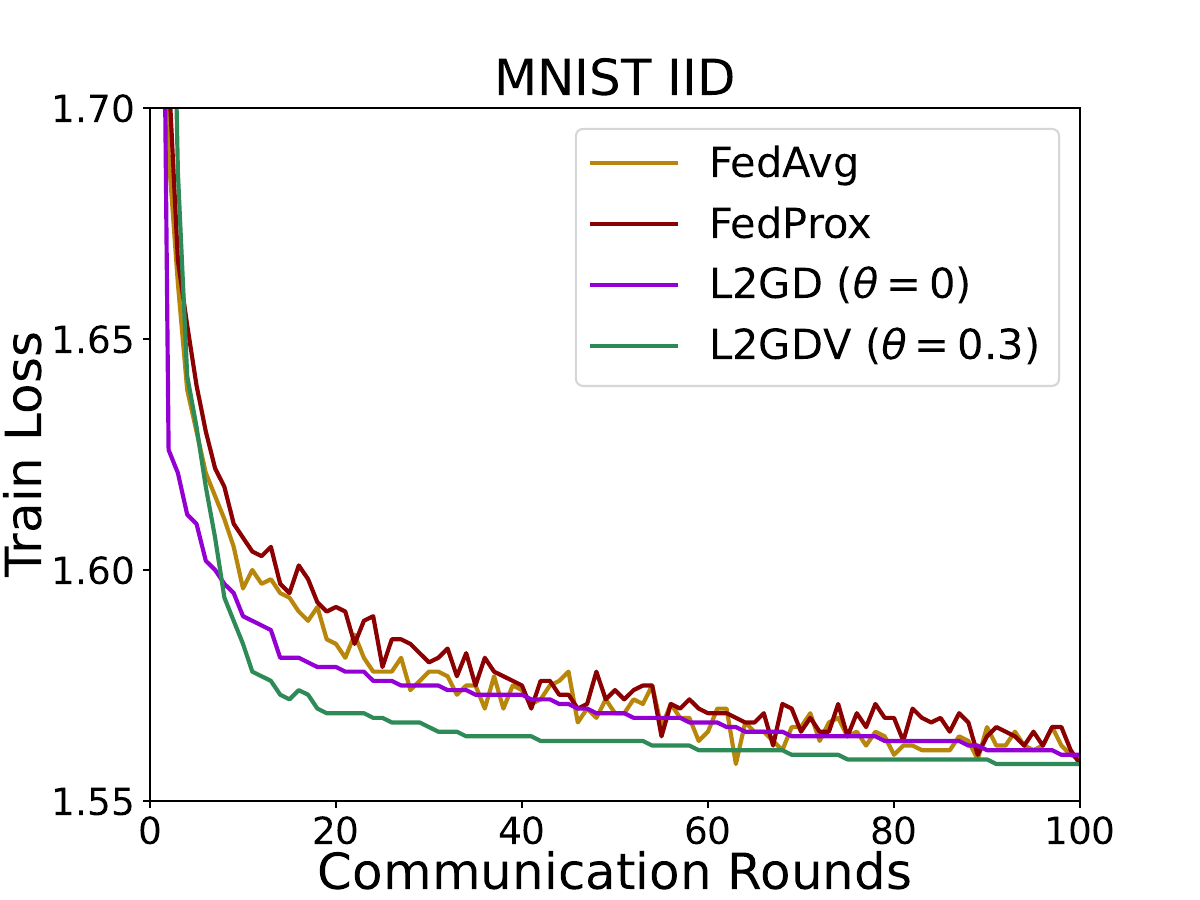}
        \label{subfig: ML-communication}
        \end{minipage}
    }
    \subfigure[Test accuracy on MNIST Non-IID.]{
        \begin{minipage}[t]{0.475\linewidth}
            \includegraphics[width=1.00\linewidth]{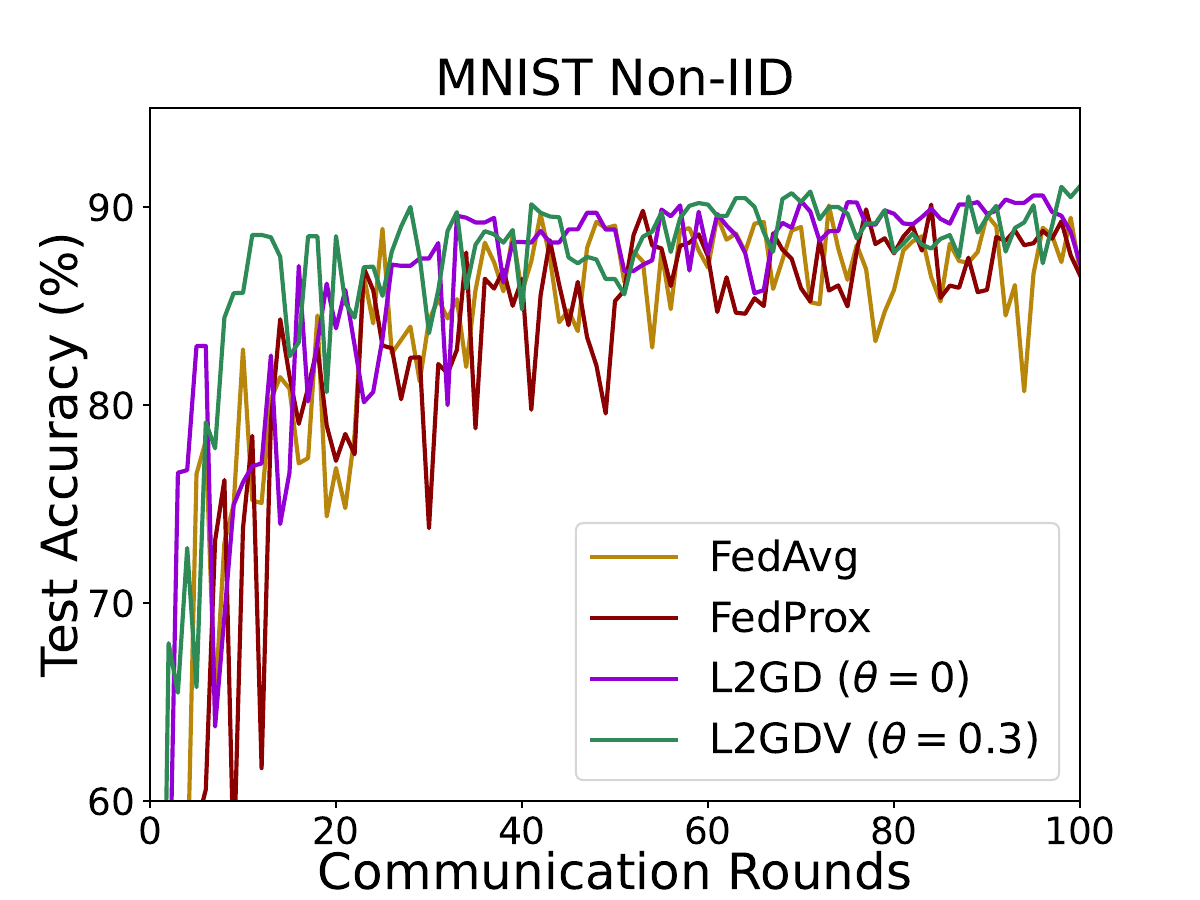}
        \label{subfig: Jester-convergence}
        \end{minipage}
    }
    \subfigure[Train loss on MNIST Non-IID.]{
        \begin{minipage}[t]{0.475\linewidth}
            \includegraphics[width=1.00\linewidth]{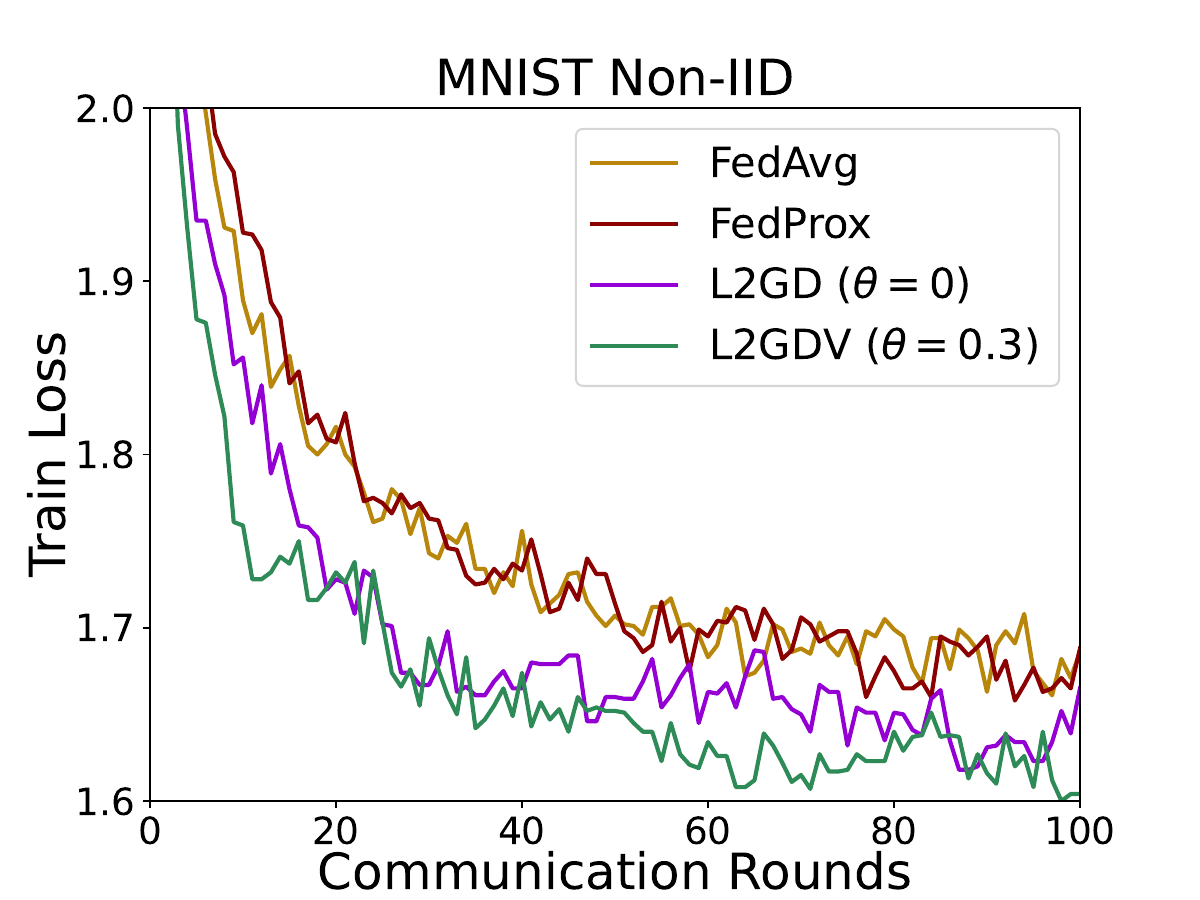}
        \label{subfig: Jester-communication}
        \end{minipage}
    }
    \caption{Test set accuracy and training set loss w.r.t. communication rounds in a convex situation (i.e., using the LR model). Figures (a) and (b) show the results for MNIST IID, while Figures (c) and (d) show those for MNIST Non-IID.}
    \label{fig:convex}
\end{figure}

In this section, we conduct comprehensive experiments to demonstrate the effectiveness of the L2GDV and verify the obtained theoretical results. 

\subsection{Experimental details}
We conduct the experiments on a well-known real dataset, MNIST~\cite{lecun1998gradient}, frequently used in prior works about federated learning. Specifically, MNIST is a dataset for classification problems, formed as hand-written digits (i.e., $0$-$9$) with $28\times28$ pixels. In addition, MNIST consists of a training set of size $60,000$ and a test set of size $10,000$. 

We utilize convolutional neural networks (CNNs) and multinomial logistic regression (LR) to study the non-convex and convex situations, respectively. 
Furthermore, we consider two partition approaches, IID and Non-IID, of the MNIST dataset, and the partition details are specified in Section 3 of~\cite{mcmahan2016federated}. We denote them as MNIST IID and MNIST Non-IID for simplicity. 
We compare the proposed L2GDV with three strong baselines, FedAvg~\cite{mcmahan2016federated}, FedProx~\cite{li2020federated1} and L2GD~\cite{hanzely2020federated}. For each baseline, we use the default optimal hyperparameters for a fair performance comparison. In particular, the number of clients is set as $100$, which means each client has $600$ samples.
Conform to the theoretical results, we take a polynomially decaying step size sequence for our proposed L2GDV, formed as $\alpha_k=\alpha_1 k^{-\theta}$, and we tune the parameter as $\theta=0.3$.  

\subsection{Experimental results}
We compare the performances by reporting the changes in the test accuracy and train loss as communication rounds increase for each situation (non-convex and convex situations) and for each partition approach (IID and Non-IID). 

\subsubsection{Non-convex situation} We report the results in a non-convex situation (i.e., using the CNN model) in Figure~\ref{fig:nonconvex}. The proposed L2GDV outperforms all the baselines in the test accuracy and train loss. Specifically, L2GDV achieved a significant lead in MNIST IID and MNIST Non-IID regarding model performances (i.e., test accuracies). In addition, the results of train losses show that L2GDV possesses faster convergence, which means it achieves its optimum within fewer communication rounds. The simulation results in the non-convex situation support the theoretical results in Section~\ref{sec:nonconvex}. In the non-convex situation, leveraging a decaying step size sequence guarantees convergence to the optimal function value with fast convergence.  

\subsubsection{Convex situation} The experimental results in a convex situation are demonstrated in Figure~\ref{fig:convex}, similar to those of the non-convex situation. Although the lead is not that large compared to that in a non-convex situation, one can easily observe that L2GDV converges faster, especially in MNIST IID. All methods have a fluctuating and unstable performance in MNIST Non-IID. However, L2GDV can provide a relatively stable training process compared to baselines, leading to better accuracy. The results verify the theorems in Section~\ref{sec:convex} that a decaying step size sequence of L2GDV is sufficient to guarantee convergence to an optimal solution in a convex situation.

\section{Conclusion}
\label{conclusion section}

In this paper, we have studied a federated learning algorithm L2GDV which is an SGD method  for solving a regularized empirical risk minimization problem involving a global error function associated with individual aggregate losses and a regularization term. A regularization parameter $\lambda$ and a probability level parameter $p$ for the stochastic estimate of the gradient can control the level of communications between the central and local machines. 
We improve the existing method by allowing flexible step sizes and carry out novel analysis for the convergence of the algorithm. Our analysis consists of two parts corresponding to a non-convex setting and a standard strongly convex setting. In the non-convex setting, we have assumed a PL condition and derived rates of convergence. In the strongly convex setting, we presented a sufficient and necessary condition for the convergence in expectation, which demonstrates that a changing step size sequence is needed for the convergence.

\section*{Acknowledgments}
The work described in this paper was partially supported by InnoHK initiative, The Government of the HKSAR, and Laboratory for AI-Powered Financial Technologies. The authors would like to thank the referees for their constructive comments and valuable suggestions. 


\newpage
\appendix
\section{Proof of Results in the Non-convex Situation}
\label{proof_nonconvex}
Before carrying out detailed convergence analysis, we present some analysis for the step iteration, illustrating how the error is reduced from $F(x^{k})-F^*$ to $\mathbb{E}_{x_k}[F(x^{k+1})-F^*]$. 

By the definition (\ref{L2GDnew}) of the sequence $\{x^k\}$, we have $x^{k+1}=x^k-\alpha_k G(x^k)$. Leveraging the $L_F$-smoothness of $F$ to be given in Lemma \ref{lemma_smooth} below, we have 
\begin{align*}
F(x^{k+1})&\le F(x^k) +\langle\nabla F(x^k),x^{k+1}-x^k\rangle + \frac{L_F}{2}\Vert x^{k+1}-x^k\Vert^2\\
&= F(x^k) +\langle\nabla F(x^k),-\alpha_k G(x^k)\rangle + \frac{L_F}{2}\alpha_k^2\Vert G(x^k)\Vert^2. 
\end{align*}
As $G(x)$ is an unbiased estimator of $\nabla F(x)$, taking expectations conditioned on $x^k$ gives 
\begin{equation}\label{onestepiteration_nonconvex}
\mathbb{E}_{x_k}[F(x^{k+1})]\le F(x^k) -\alpha_k\Vert\nabla F(x^k)\Vert^2 + \frac{L_F}{2}\alpha_k^2\mathbb{E}_{x_k}\Vert G(x^k)\Vert^2.
\end{equation}
The middle term can be bounded by applying the definition (\ref{PLdef}) of the PL condition as 
$$
\Vert\nabla F(x^k)\Vert^2\ge 2\mu(F(x^k)-F^*).
$$
Hence  
\begin{equation}\label{PL_est}
\mathbb{E}_{x_k}[F(x^{k+1})-F^*]\le (1-2\alpha_k\mu)[F(x^{k})-F^*] + \frac{L_F}{2}\alpha_k^2\mathbb{E}_{x_k}\Vert G(x^k)\Vert^2.
\end{equation} 
This one step iteration result enables us to derive convergence and rates of convergence from 
$$\Pi_{i=1}^k \left(1- 2\alpha_i \mu\right) \leq \Pi_{i=1}^k \exp\left\{- 2\alpha_i \mu\right\} =  \exp\left\{- 2\left(\sum_{i=1}^k \alpha_i\right) \mu\right\} $$
after the PL condition exponent $\mu$ is found and the minor term $\frac{L_F}{2}\alpha_k^2\mathbb{E}_{x_k}\Vert G(x^k)\Vert^2$ is estimated.

\subsection{General analysis for the non-convex situation}
To conduct detailed analysis with the result (\ref{PL_est}) of the one-step iteration, we need the PL condition exponent $\mu$, the smooth exponent $L_F$, and estimates of the minor term $\alpha_k^2\mathbb{E}_{x_k}\Vert G(x^k)\Vert^2$. 
\begin{lemma}
\label{lemma_smooth}
Under Assumption~\ref{assumption_0}, $f$ is $L_f$-smooth with $L_f=\frac{L}{n}$, $\psi$ is $L_\psi$-smooth with $L_\psi=\frac{1}{n}$, $F$ is $L_F$-smooth with $L_F = \frac{L+\lambda}{n}$.
\end{lemma}
\begin{proof} Observe 
\begin{equation}\label{nablaf}
\nabla f(x)=\frac{1}{n}(\nabla f_1(x_1),\nabla f_2(x_2),\cdots,\nabla f_n(x_n))^T, 
\end{equation} 
and 
$$ \nabla^2 f(x)=\frac{1}{n}\text{diag}(\nabla^2 f_1(x_1),\nabla^2 f_2(x_2),\cdots,\nabla^2 f_n(x_n)). $$
Then by Assumption~\ref{assumption_0}, we have 
$$ \nabla^2 f_i(x_i)\preceq L I_d,$$
where $I_d$ as the $d\times d$ identity matrix, which implies 
$$ \nabla^2 f(x)\preceq \frac{L}{n} I_{nd}$$
and thereby $f$ is $\frac{L}{n}$-smooth. Notice for the function $\psi$ that 
\begin{equation}\label{nablapsi}
\nabla\psi(x)=\frac{1}{n}(x_1-\bar{x}, x_2-\bar{x},\cdots,x_n-\bar{x})^\mathrm {T}.
\end{equation} 
If we denote $e$ as the all one vector in $\mathbb{R}^{n}$, then we have 
\[\nabla^2\psi(x)=\frac{1}{n}(I_n-\frac{1}{n}ee^T),\]
where $I_n-\frac{1}{n}ee^T$ is a circulant matrix with each eigenvalue either zero or one. Hence the largest eigenvalue of $\nabla^2\psi(x)$ is $\frac{1}{n}$ and 
$\nabla^2\psi(x)\preceq\frac{1}{n} I_n$, telling us that $\psi$ is $\frac{1}{n}$-smooth. Thus, $F$ is $L_F$-smooth with $L_F = \frac{L+\lambda}{n}$.
\end{proof}

\begin{lemma}
\label{lemma_QG}
Let $x_p$ be the projection of $x$ onto the solution set $\mathcal{X}^*$. If $g$ satisfies a PL condition with some $\mu>0$, then we have 
$$g(x)-g^*\ge\frac{\mu}{2}\Vert x-x_p\Vert^2,  \qquad \forall x\in \mathbb{R}^d.$$
\end{lemma}

\begin{lemma}
\label{lemma_0}
Under Assumption~\ref{assumption_0}, for every $x\in \mathbb{R}^d$, the stochastic gradient $G$ of $F$ satisfies 
\[\mathbb{E}[\Vert G(x)-G(x_p)\Vert^2]\le \frac{2\zeta}{\mu}(F(x)-F(x_p))+8\sigma_{m}^2\]
and 
\[\mathbb{E}[\Vert G(x)\Vert^2]\le \frac{4\zeta}{\mu}(F(x)-F(x_p))+18\sigma_{m}^2.\]
\end{lemma}
\begin{proof}
By the definition of the stochastic gradient $G$, the expectation $\mathbb{E}[\Vert G(x)-G(x_p)\Vert^2] = \mathbb{E}_{x, x_p} [\Vert G(x)-G(x_p)\Vert^2]$ equals 
\begin{align*}
& (1-p)p\left\|\frac{\nabla f(x)}{1-p}-\lambda\frac{\nabla\psi(x_p)}{p}\right\|^2+p(1-p)\left\|\lambda\frac{\nabla \psi(x)}{p}-\frac{\nabla f(x_p)}{1-p}\right\|^2\\
&=(1-p)^2\Vert I_1-I_2\Vert^2+p^2\Vert I_3-I_4\Vert^2
+(1-p)p\Vert I_1-I_4\Vert^2+p(1-p)\Vert I_3-I_2\Vert^2, 
\end{align*}
where for simplicity, we denote $I_1=\frac{\nabla f(x)}{1-p},I_2=\frac{\nabla f(x_p)}{1-p},I_3=\lambda\frac{\nabla \psi(x)}{p},I_4=\lambda\frac{\nabla\psi(x_p)}{p}$. 
Expressing $(1-p)^2\Vert I_1-I_2\Vert^2$ as $(1-p) \Vert I_1-I_2\Vert^2 - (1-p)p \Vert I_1-I_2\Vert^2$ and the term $p^2\Vert I_3-I_4\Vert^2$ in the same way, we know that $\mathbb{E}[\Vert G(x)-G(x_p)\Vert^2]$ equals  
\begin{eqnarray*}
&&(1-p)\Vert I_1-I_2\Vert^2+p\Vert I_3-I_4\Vert^2 \\
&&+(1-p)p\left(\Vert I_1-I_4\Vert^2+\Vert I_3-I_2\Vert^2 - \Vert I_1-I_2\Vert^2 - \Vert I_3-I_4\Vert^2\right) \\
&=&(1-p)\Vert I_1-I_2\Vert^2+p\Vert I_3-I_4\Vert^2 \\
&&+2(1-p)p(\langle I_1,I_2\rangle+\langle I_3,I_4\rangle-\langle I_1,I_4\rangle-\langle I_2,I_3\rangle).
\end{eqnarray*}
Observe that $-2 \langle I_1,I_4\rangle \leq \Vert I_1\Vert^2 + \Vert I_4\Vert^2$ and the same bound holds for $-2\langle I_2,I_3\rangle$. We have 
\begin{eqnarray*}
\mathbb{E}[\Vert G(x)-G(x_p)\Vert^2]&\leq&(1-p)\Vert I_1-I_2\Vert^2+p\Vert I_3-I_4\Vert^2 \\
&&+(1-p)p\left(\Vert I_1+I_2\Vert^2+\Vert I_3+I_4\Vert^2\right).
\end{eqnarray*}
Expanding $I_1+I_2$ as $I_1-I_2 + 2 I_2$, we bound $\Vert I_1+I_2\Vert^2$ by $2\Vert I_1-I_2\Vert^2 + 2 \Vert 2 I_2\Vert^2$ and do the same for the term $\Vert I_3+I_4\Vert^2$. It follows that 
\begin{eqnarray*}
\mathbb{E}[\Vert G(x)-G(x_p)\Vert^2]&\leq&(1-p)(1+2p) \Vert I_1-I_2\Vert^2+p(3-2p) \Vert I_3-I_4\Vert^2  \\
&&+8(1-p)p\left(\Vert I_2\Vert^2+\Vert I_4\Vert^2\right).
\end{eqnarray*}
To estimate the norm squares $\Vert I_2\Vert^2+\Vert I_4\Vert^2$, we apply the explicit formulae (\ref{nablaf}), (\ref{nablapsi}) for $\nabla f(x_p), \nabla\psi(x_p)$ and find 
\begin{eqnarray*}
\Vert I_2\Vert^2+\Vert I_4\Vert^2 &=& \frac{1}{(1-p)^2 n^2} \sum_{i=1}^n \left\|\nabla f_i(x_{p,i})\right\|^2 + \frac{\lambda^2}{p^2 n^2} \sum_{i=1}^n \left\| x_{p,i}-\bar{x}_p\right\|^2 \\
&\leq& \frac{\sigma_{x_p}^2}{p(1-p)}\leq\frac{\sigma_{m}^2}{p(1-p)}. 
\end{eqnarray*}
Hence 
\begin{eqnarray*}
\mathbb{E}[\Vert G(x)-G(x_p)\Vert^2]\leq(1-p)(1+2p) \Vert I_1-I_2\Vert^2+p(3-2p) \Vert I_3-I_4\Vert^2  +8\sigma_{m}^2. 
\end{eqnarray*}
We apply Lemma \ref{lemma_smooth} and the equivalence of smoothness and Lipschitz continuous gradient to the functions $f$ and $\psi$, and find 
$$ \Vert I_1-I_2\Vert^2 = \frac{1}{(1-p)^2} \left\|\nabla f(x) -\nabla f(x_p)\right\|^2 \leq \frac{L_f^2}{(1-p)^2} \left\|x -x_p\right\|^2 $$
and 
$$ \Vert I_3-I_4\Vert^2 = \frac{\lambda^2}{p^2} \left\|\nabla \psi(x) -\nabla \psi(x_p)\right\|^2 \leq \frac{\lambda^2L_{\psi}^2}{p^2} \left\|x -x_p\right\|^2. $$
It follows that 
$$\mathbb{E}[\Vert G(x)-G(x_p)\Vert^2]\leq \Big(\frac{1+2p}{1-p} L_f^2+\frac{(3-2p)\lambda^2}{p}L_{\psi}^2\Big) \left\|x -x_p\right\|^2 +8\sigma_{m}^2. 
$$
Applying Lemma \ref{lemma_QG} and letting $\zeta = \frac{1+2p}{1-p} L_f^2+\frac{(3-2p)\lambda^2}{p}L_{\psi}^2$, we can continue our analysis as 
\begin{align*}
\mathbb{E}[\Vert G(x)-G(x_p)\Vert^2]&\le\zeta \left\|x -x_p\right\|^2 +8\sigma_{m}^2\\
&\le \frac{2\zeta}{\mu}(F(x)-F^*)+8\sigma_{m}^2.
\end{align*}
This proves the first desired estimate. 

For the second estimate, we express $G(x)= \left\{G(x)-G(x_p)\right\} + G(x_p)$ and bound 
$\Vert G(x)\Vert^2$ by $2 \Vert G(x)-G(x_p)\Vert^2 + 2 \Vert G(x_p)\Vert^2$. Then the desired estimate follows from the expression 
\begin{align*}
\mathbb{E}[\Vert G(x_p)\Vert^2] &= (1-p)\left\|\frac{\nabla f(x_p)}{1-p}\right\|^2+p\left\|\frac{\lambda\nabla \psi(x_p)}{p}\right\|^2 \\
&=\frac{1}{n^2}\sum_{i=1}^{n}\left(\frac{1}{1-p}\Vert\nabla f_i(x_{p,i})\Vert^2+\frac{\lambda^2}{p}\Vert x_{p,i}-\bar{x}_p\Vert^2\right) =
\sigma_{x_p}^2\le\sigma_{m}^2.
\end{align*}
The proof of Lemma \ref{lemma_0} is complete. 
\end{proof}

\subsection{Deriving convergence rates with a fixed step size}

Now we can prove Theorem~\ref{theorem_PL_fixed}.

\begin{proof}[Proof of Theorem~\ref{theorem_PL_fixed}] 
As we have the smoothness in Lemma~\ref{lemma_smooth} and the PL condition of the function $F$, we can follow the estimate of the upper bound in \eqref{PL_est} to find 
$$\mathbb{E}[F(x^{k+1})-F^*]\le (1-2\alpha\mu)[F(x^{k})-F^*] + \frac{L_F}{2}\alpha^2\mathbb{E}\Vert G(x^k)\Vert^2.$$
Then we plug the estimate of $\mathbb{E}\Vert G(x^k)\Vert^2$ of Lemma \ref{lemma_0} to obtain
\begin{align}
\label{one_iter_bound_PL}
    \mathbb{E}[F(x^{k+1})-F^*] &\le (1-2\alpha\mu+\frac{2\zeta L_F\alpha^2}{\mu})[F(x^{k})-F^*] + 9L_F\alpha^2\sigma_{m}^2\nonumber\\
    &\le (1-\alpha\mu)[F(x^{k})-F^*] + 9L_F\alpha^2\sigma_{m}^2,
\end{align}
where the last inequality is obtained by letting $\alpha\le\frac{\mu^2}{2\zeta L_F}$. Applying this one-step estimate recursively, we have 
\begin{align*}
    \mathbb{E}[F(x^{k+1})-F^*] &\le (1-\alpha\mu)^k[F(x^{1})-F^*] + 9L_F\alpha^2\sigma_{m}^2\sum_{i=0}^{k-1}(1-\alpha\mu)^i\\
    &\le (1-\alpha\mu)^k[F(x^{1})-F^*] + \frac{9L_F\alpha\sigma_{m}^2}{\mu}.
\end{align*}
This completes the proof of Theorem~\ref{theorem_PL_fixed}. 
\end{proof}

\subsection{Deriving convergence rates with decaying step sizes}

We turn to the proof of Theorem~\ref{theorem_PL_decay}. We need the following elementary inequality which can be found in ~\cite{smale2009online}. 

\begin{lemma}
\label{lemma_4}
For any $v,a>0$, there holds
\[e^{-vx}\le(\frac{a}{ve})^a x^{-a}, \qquad \forall x>0.\]
\end{lemma}

\begin{proof}[Proof of Theorem~\ref{theorem_PL_decay}] 
Following the upper bound (\ref{one_iter_bound_PL}) obtained in the proof of Theorem~\ref{theorem_PL_fixed} which is formulated with $\alpha = \alpha_k$ as
\begin{align*}
    \mathbb{E}[F(x^{k+1})-F^*] \le (1-\alpha_k\mu)[F(x^{k})-F^*] + 9L_F\alpha_k^2\sigma_{m}^2. 
\end{align*}
We know from the choice $\alpha_k=\alpha_1 k^{-\theta}\le\frac{1}{2\mathcal{L}}$ that the residual error  $R^{k+1}:=F(x^{k+1})-F^*$ can be estimated as 
\begin{eqnarray*}
\mathbb{E}[R^{k+1}] &\le& \prod_{i=1}^{k}\left(1-\frac{\mu \alpha_1}{i^{\theta}}\right) R^{1} \\
&& +9L_F\sigma_{m}^2 \sum_{i=1}^{k-1}\prod_{j=i+1}^k \left(1-\frac{\mu \alpha_1}{j^{\theta}}\right) \frac{\alpha_1^2}{i^{2\theta}}+ \frac{9L_F\sigma_{m}^2 \alpha_1^2}{k^{2\theta}}.
\end{eqnarray*}

The first term in the above bound is easy to deal with. We have 
\begin{align*}
&\prod_{i=1}^{k}\left(1-\frac{\mu \alpha_1}{i^{\theta}}\right) 
\le\exp\left\{-\sum_{i=1}^{k}\frac{\mu \alpha_1}{i^{\theta}}\right\}\le\exp\left\{-\mu \alpha_1\int_{1}^{k+1} x^{-\theta}dx \right\}. 
\end{align*} 
So, when $0<\theta<1$, we have $\int_{1}^{k+1} x^{-\theta}dx = \frac{(k+1)^{1-\theta}-1}{1-\theta}$ and then apply Lemma~\ref{lemma_4} with $a=\frac{1}{1-\theta}, v= \frac{\mu \alpha_1}{1-\theta}$ to obtain 
$$ \prod_{i=1}^{k}\left(1-\frac{\mu \alpha_1}{i^{\theta}}\right) 
\le\exp\Bigg\{\frac{\mu \alpha_1}{1-\theta}\Bigg\}\left(\frac{1}{\mu \alpha_1 e}\right)^\frac{1}{1-\theta}k^{-1}. $$ 
When $\theta=1$, we also have $\int_{1}^{k+1} x^{-\theta}dx = \log (k+1)$ and then 
$$ \prod_{i=1}^{k}\left(1-\frac{\mu \alpha_1}{i^{\theta}}\right) 
\le (k+1)^{-\mu \alpha_1}. $$ 

The second term without a constant factor can be bounded first as 
\begin{align*}
&\sum_{i=1}^{k-1}\prod_{j=i+1}^k \left(1-\frac{\mu \alpha_1}{j^{\theta}}\right) \frac{1}{i^{2\theta}} \leq 
\sum_{i=1}^{k-1} \frac{1}{i^{2\theta}} \exp\Bigg\{-\mu \alpha_1\sum_{j=i+1}^{k} j^{-\theta}\Bigg\} = I_1 + I_2:\\
&=\sum_{k/2\le i\le k} \frac{1}{i^{2\theta}} \exp\Bigg\{-\mu \alpha_1\sum_{j=i+1}^{k} j^{-\theta}\Bigg\} +\sum_{1\le i< k/2} \frac{1}{i^{2\theta}} \exp\Bigg\{-\mu \alpha_1\sum_{j=i+1}^{k} j^{-\theta}\Bigg\}.
\end{align*} 
Then we bound $I_1$ and $I_2$ separately. 

Consider the case $0<
\theta<1$. For $I_1$, we observe that 
$$\sum_{j=i+1}^k j^{-\theta}\ge\int_{i+1}^{k+1}x^{-\theta}dx= \frac{(k+1)^{1-\theta}}{1-\theta}\left(1-\left(\frac{i+1}{k+1}\right)^{1-\theta}\right) $$
and the elementary inequality $(1-x)^{1-\theta}\le 1-(1-\theta)x$ for $0\leq x\leq 1$ yields 
$$\left(\frac{i+1}{k+1}\right)^{1-\theta} = \left(1-\frac{k-i}{k+1}\right)^{1-\theta} \leq 1- (1-\theta) \frac{k-i}{k+1}, $$
we have 
\begin{align*}
I_1&\le\sum_{k/2\le i\le k} \left(\frac{2}{k}\right)^{2\theta} \exp\Bigg\{-{\mu \alpha_1(k+1)^{-\theta}}(k-i)\Bigg\}\\
&\le \left(\frac{2}{k}\right)^{2\theta} \left\{1+  \int_0^{\frac{k}{2}}\exp\Bigg\{-{\mu \alpha_1 (k+1)^{-\theta}}x\Bigg\}dx\right\} \\
&\le  \left(\frac{2}{k}\right)^{2\theta} \left\{1+ \frac{(k+1)^\theta}{\mu \alpha_1}\right\}\le\frac{2^{3\theta}}{\mu \alpha_1}k^{-\theta}. 
\end{align*}
For $I_2$, we notice that $\frac{i+1}{k+1} \leq \frac{(k-1)/2 +1}{k+1} \leq \frac{1}{2}$ for $1 \leq i < k/2$ and $1/i^{2\theta} \leq 1$, and thereby 
$$
I_2 \le \frac{k}{2} \exp\Bigg\{-\frac{\mu \alpha_1 (1-2^{\theta-1})}{1-\theta}(k+1)^{1-\theta}\Bigg\}$$
which in connection with Lemma~\ref{lemma_4} with $a=\frac{2}{1-\theta}$ and $v=\frac{\mu \alpha_1 (1-2^{\theta-1})}{1-\theta}$ yields 
$$ I_2 \leq  \frac{k}{2}\left(\frac{2}{\mu \alpha_1 e(1-2^{\theta-1})}\right)^{\frac{2}{1-\theta}}(k+1)^{-2} \le \frac{1}{2}\left(\frac{2}{\mu \alpha_1 e(1-2^{\theta-1})}\right)^{\frac{2}{1-\theta}} k^{-1}.
$$
Combining the above estimates, we have 
$$
\mathbb{E}[R^{k+1}]\le C_{\mu, \alpha_1, \theta, R^{1}} k^{-\theta}$$
where $C_{\mu, \alpha_1, \theta, R^{1}}$ is a constant independent of $k$ given by 
\begin{align*} C_{\mu, \alpha_1, \theta, R^{1}} &= \exp\Bigg\{\frac{\mu \alpha_1}{1-\theta}\Bigg\}\left(\frac{1}{\mu \alpha_1 e}\right)^\frac{1}{1-\theta}R^1 \\
&+ 9L_F\sigma_{m}^2 \alpha_1^2 \left(\frac{2^{3\theta}}{\mu \alpha_1} + \frac{1}{2}\left(\frac{2}{\mu \alpha_1 e(1-2^{\theta-1})}\right)^{\frac{2}{1-\theta}} +1\right). 
\end{align*}
This verifies the desired rate of convergence (\ref{ratesofconvergence}) in the case $0<\theta <1$. 

For the case $\theta=1$, we also have 
\begin{align*} 
I_1 + I_2 &\leq \sum_{i=1}^{k-1} \frac{1}{i^{2}} \exp\Bigg\{-\mu \alpha_1 \int_{i+1}^{k+1}\frac{1}{x}dx \Bigg\} \\
&= (k+1)^{-\mu \alpha_1} \sum_{i=1}^{k-1} \frac{(i+1)^{\mu \alpha_1}}{i^2} 
\leq 2^{\mu \alpha_1} (k+1)^{-\mu \alpha_1} \sum_{i=1}^{k-1} i^{\mu \alpha_1 -2}.  
\end{align*}
But 
$$ \sum_{i=1}^{k-1} i^{\mu \alpha_1 -2} \leq \left\{\begin{array}{ll} 
1 + \int_1^\infty x^{\mu \alpha_1 -2} d x \leq 1 + \frac{1}{1-\mu \alpha_1}, & \hbox{if}  \ \mu \alpha_1 <1, \\
1 + \int_1^k \frac{1}{x} d x \leq 1 + \log k, & \hbox{if}  \ \mu \alpha_1 =1, \\
1 + \int_1^{k} x^{\mu \alpha_1 -2} d x \leq 1 + \frac{k^{\mu \alpha_1 -1} - 1}{\mu \alpha_1 -1}, & \hbox{if}  \ \mu \alpha_1 > 1, 
\end{array}\right. $$
Therefore, we have 
$$
\mathbb{E}[R^{k+1}]\le C_{\mu, \alpha_1, \theta, R^{1}, n} 
\left\{\begin{array}{ll} 
k^{-\mu \alpha_1}, & \hbox{if}  \ \mu \alpha_1 <1, \\
\frac{1 + \log k}{k}, & \hbox{if}  \ \mu \alpha_1 =1, \\
\frac{1}{k}, & \hbox{if}  \ \mu \alpha_1 > 1. 
\end{array}\right. $$
where the constant $C_{\mu, \alpha_1, \theta, R^{1}, n}$ is independent of $k$ given by 
\begin{align*} C_{\mu, \alpha_1, \theta, R^{1}} &=  R^{1} + 18\sigma_{m}^2 \alpha_1^2 
+ 18\sigma_{m}^2 \alpha_1^2 2^{\mu \alpha_1} \left\{\begin{array}{ll} 
\frac{2- \mu \alpha_1}{1-\mu \alpha_1}, & \hbox{if}  \ \mu \alpha_1 <1, \\
1, & \hbox{if}  \ \mu \alpha_1 =1, \\
\frac{\mu \alpha_1}{\mu \alpha_1 -1}, & \hbox{if}  \ \mu \alpha_1 > 1. 
\end{array}\right.
\end{align*}
This verifies the desired rate of convergence (\ref{ratesofconvergence_PL}) for the case $\theta=1$ and completes the proof of Theorem~\ref{theorem_PL_decay}. 
\end{proof}

\section{Proof of Results in the Convex Situation}
\label{proof_convex}

To illustrate further our ideas of the proof, we denote $x^*=x(\lambda)$ and the residual $r^k=x^k-x^*$. 
By the definition (\ref{L2GDnew}) of the sequence $\{x^k\}$, we have $r^{k+1}=r^k-\alpha_k G(x^k)$. Taking inner products with themselves on $\mathbb{R}^d$ for both sides yields 
\begin{align*}
\Vert r^{k+1}\Vert^2 =\Vert r^k\Vert^2 -2\alpha_k \langle r^k,G(x^k) \rangle + \alpha_k^2\Vert G(x^k)\Vert^2.
\end{align*}
As $G(x)$ is an unbiased estimator of $\nabla F(x)$, taking expectations conditioned on $x^k$ gives 
\begin{equation}\label{onestepiteration}
\mathbb{E}_{x_k}\Vert r^{k+1}\Vert^2 = \Vert r^k\Vert^2 -2\alpha_k \langle r^k,\nabla F(x^k) \rangle + \alpha_k^2\mathbb{E}_{x_k}\Vert G(x^k)\Vert^2.
\end{equation}
The core of our analysis here is on the middle term 
\begin{equation}\label{middletermexpress}
-2\alpha_k \langle r^k,\nabla F(x^k) \rangle = 2\alpha_k \nabla F(x^k)^T \left(x^* -x^k\right). 
\end{equation}
This term can be bounded by applying the definition (\ref{strongconvexdef}) of the $\mu_F$-strongly convex $F$ as 
$$
-2\alpha_k \langle r^k,\nabla F(x^k) \rangle \leq 2\alpha_k \left(F(x^*) - F(x^k) - \frac{\mu_F}{2}\Vert x^* -x^k\Vert^2\right). 
$$
Hence  
\begin{equation}\label{middletermest}
\mathbb{E}_{x_k}\Vert r^{k+1}\Vert^2 \leq \left(1- \alpha_k \mu_F\right) \Vert r^k\Vert^2 -2\alpha_k \left(F(x^k) - F(x^*)\right)+ \alpha_k^2\mathbb{E}_{x_k}\Vert G(x^k)\Vert^2
\end{equation} 
and convergence and rates of convergence can be obtained from 
$$\Pi_{i=1}^k \left(1- \alpha_i \mu_F\right) \leq \Pi_{i=1}^k \exp\left\{- \alpha_i \mu_F\right\} =  \exp\left\{- \left(\sum_{i=1}^k \alpha_i\right) \mu_F\right\} $$
after the strongly convex exponent $\mu_F$ is found and the minor term $\alpha_k^2\mathbb{E}_{x_k}\Vert G(x^k)\Vert^2$ is estimated. 

A lower bound for the middle term (\ref{middletermexpress}) of the one-step iteration (\ref{onestepiteration}), 
to derive the necessity of convergence, can be obtained by taking the sum of (\ref{smoothdef}) applied to the
$L_F-$smooth function $F$ with the two pairs $(x, y) = (x^k, x^*)$ and $(x^*, x^k)$: 
$$ 0 \leq \left\{\nabla F(x^*)^T (x^k -x^*)+\frac{L_F}{2}\Vert r^k\Vert^2\right\} + \left\{\nabla F(x^k)^T (x^* -x^k)+\frac{L_F}{2}\Vert r^k\Vert^2\right\}
$$
which, together with (\ref{onestepiteration}) and the identity $\nabla F(x^*) =0$, leads to 
\begin{equation}\label{middletermestlower}
\mathbb{E}_{x_k}\Vert r^{k+1}\Vert^2 \geq \left(1- 2 \alpha_k L_F\right) \Vert r^k\Vert^2 + \alpha_k^2\mathbb{E}_{x_k}\Vert G(x^k)\Vert^2. 
\end{equation} 

\subsection{Convexity analysis}

To carry out the detailed analysis with the upper bound (\ref{middletermest}) and the lower one (\ref{middletermestlower}) of the one-step iteration, we need the strongly convex exponent $\mu_F$, the smooth exponent $L_F$, and refined estimates of the minor term $\alpha_k^2\mathbb{E}_{x_k}\Vert G(x^k)\Vert^2$. 

Recall that $F= f + \lambda \psi$ and $f(x)=\frac{1}{n}\sum_{i=1}^n f_i(x_i), \psi(x)=\frac{1}{2n}\sum_{i=1}^n \Vert x_i-\bar{x}\Vert ^2.$

\begin{lemma}
\label{lemma_2}
Under Assumption~\ref{assumption_1}, $f$ is $L_f$-smooth with $L_f=\frac{L}{n}$, $\mu_f$-strongly convex with $\mu_f=\frac{\mu}{n}$, $\psi$ is convex and $L_\psi$-smooth with $L_\psi=\frac{1}{n}$, and $F$ is $\frac{\mu}{n}$-strongly convex.
\end{lemma}
\begin{proof} Following the proof of Lemma 1, we observe 
\begin{equation}
\nabla f(x)=\frac{1}{n}(\nabla f_1(x_1),\nabla f_2(x_2),\cdots,\nabla f_n(x_n))^T, 
\end{equation} 
and 
$$ \nabla^2 f(x)=\frac{1}{n}\text{diag}(\nabla^2 f_1(x_1),\nabla^2 f_2(x_2),\cdots,\nabla^2 f_n(x_n)). $$
Then by Assumption~\ref{assumption_1}, we have 
$$ \mu I_d\preceq\nabla^2 f_i(x_i)\preceq L I_d $$
which implies 
$$ \frac{\mu}{n} I_{nd}\preceq\nabla^2 f(x)\preceq \frac{L}{n} I_{nd}$$
and thereby $f$ is $\frac{L}{n}$-smooth and $\frac{\mu}{n}$-strongly convex. 

As in the proof of Lemma 1, we find the Hessian of the function $\psi$ to be 
\[\nabla^2\psi(x)=\frac{1}{n}(I_n-\frac{1}{n}ee^T),\]
where $I_n-\frac{1}{n}ee^T$ is a circulant matrix with each eigenvalue either zero or one. Thus, 
$0\preceq\nabla^2\psi(x)\preceq\frac{1}{n} I$, telling us that $\psi$ is convex and $\frac{1}{n}$-smooth. It follows that $F$ is $\frac{\mu}{n}$-strongly convex.

\end{proof}

To refine the previous estimate of the minor term $\alpha_k^2\mathbb{E}_{x_k}\Vert G(x^k)\Vert^2$ under Assumption~\ref{assumption_1}, we need the concept of Bregman distance. 

\begin{definition}
Let $h: \Omega\to \mathbb{R}$ be a continuously differentiable, strictly convex function defined on a closed convex set $\Omega$. The associated {\bf Bregman distance}
 between $x, y \in \Omega$ is defined by  
$$D_h(x,y)=h(x)-h(y)-\langle\nabla h(y),x-y\rangle. $$
\end{definition}

The following lower bound for the Bregman distance associated with a smooth convex function is well known. We give a proof for completeness. 
    
\begin{lemma}
\label{lemma_1}
If $g: \mathbb{R}^d\to \mathbb{R}$ is convex and $L_g$-smooth, then
\[D_g(x,y)\ge \frac{1}{2 L_g}\Vert\nabla g(x)-\nabla g(y)\Vert^2, \qquad\forall x,y\in \mathbb{R}^d \]
and 
$$ \Vert\nabla g(x)-\nabla g(y)\Vert \leq L_g \Vert x-y\Vert, \qquad\forall x,y\in \mathbb{R}^d.
$$
\end{lemma}
\begin{proof}
For any $y\in \mathbb{R}^d$, let $\phi: \mathbb{R}^d \to \mathbb{R}$ be given by $\phi(x)=g(x)-\langle\nabla g(y),x\rangle$. It is easy to see that for $x, u\in \mathbb{R}^d$, 
$$\phi(x+u) - \phi(x) =g(x+u) - g(x)-\langle\nabla g(y), u\rangle $$
and $\nabla \phi(x) = \nabla g(x)- \nabla g(y)$. But $g: \mathbb{R}^d\to \mathbb{R}$ is $L_g$-smooth. Hence 
$$ g(x+u) - g(x) \leq \nabla g(x)^T u + \frac{L_g}{2}\Vert u\Vert^2 $$ and thereby 
\[\phi(x+u)\le\phi(x)+\nabla\phi(x)^T u + \frac{L_g}{2}\Vert u\Vert^2.\]
Thus, $\phi$ is $L_g$-smooth. Moreover, the convexity of $g$ also tells us that 
$$ \phi(x)-\phi(y)=g(x)-g(y)-\langle\nabla g(y),x-y\rangle \geq 0$$
which means $\phi(y)$ is the minimum value of $\phi$.
Combining this with the $L_g$-smoothness of $g$ applied to 
$t=\nabla\phi(x)=\nabla g(x)- \nabla g(y)$ yields 
\begin{align*}
\phi(y)&\le\phi(x-\frac{1}{L_g}t)\\
&\le\phi(x)-\frac{1}{L_g}\nabla\phi(x)^T t + \frac{L_g}{2}\Vert \frac{1}{L_g}t\Vert^2\\
&=\phi(x)-\frac{1}{2L_g}\Vert t\Vert^2.
\end{align*}
Together with the identity $\phi(x)-\phi(y)=D_g(x,y)$, this proves the first desired inequality. The second one follows from the definition (\ref{smoothdef}) of the $L_g-$smoothness and  
the Bregman distance $D_g(x,y)$. This completes the proof of Lemma~\ref{lemma_1}.
\end{proof} 


Recall the notation $F=f+ \lambda \psi, \mathcal{L}=\frac{1}{n}
\max\{\frac{(1+2p)L}{1-p},\frac{(3-2p)\lambda}{p}\}$ and $\sigma^2=\frac{1}{n^2}\sum_{i=1}^{n}(\frac{1}{1-p}\Vert\nabla f_i(x_i(\lambda))\Vert^2+\frac{\lambda^2}{p}\Vert x_i(\lambda)-\bar{x}(\lambda)\Vert^2)$.

\begin{lemma}
\label{lemma_3} 
Under Assumption~\ref{assumption_1}, for every $x\in \mathbb{R}^d$, the stochastic gradient $G$ of $F$ satisfies 
\[\mathbb{E}[\Vert G(x)-G(x(\lambda))\Vert^2]\le 2\mathcal{L}(F(x)-F(x(\lambda)))+8\sigma^2\]
and 
\[\mathbb{E}[\Vert G(x)\Vert^2]\le 4\mathcal{L}(F(x)-F(x(\lambda)))+18\sigma^2.\]
\end{lemma}

\begin{proof}
As in the proof of Lemma \ref{lemma_0}, the expectation $\mathbb{E}[\Vert G(x)-G(x(\lambda))\Vert^2] = \mathbb{E}_{x, x(\lambda)} [\Vert G(x)-G(x(\lambda))\Vert^2]$ equals 
$$(1-p)^2\Vert I_1-I_2\Vert^2+p^2\Vert I_3-I_4\Vert^2
+(1-p)p\Vert I_1-I_4\Vert^2+p(1-p)\Vert I_3-I_2\Vert^2,$$
where $\{I_i\}_{i=1}^4$ is the same as in the proof of Lemma \ref{lemma_0} except that $x_p$ is replaced by $x(\lambda)$. 
The same procedure yields 
\begin{eqnarray*}
\mathbb{E}[\Vert G(x)-G(x(\lambda))\Vert^2]&\leq&(1-p)(1+2p) \Vert I_1-I_2\Vert^2+p(3-2p) \Vert I_3-I_4\Vert^2  \\
&&+8(1-p)p\left(\Vert I_2\Vert^2+\Vert I_4\Vert^2\right).
\end{eqnarray*}
To estimate the norm squares $\Vert I_2\Vert^2+\Vert I_4\Vert^2$, we apply the explicit formulae (\ref{nablaf}), (\ref{nablapsi}) for $\nabla f(x(\lambda)), \nabla\psi(x(\lambda))$ and find 
\begin{eqnarray*}
\Vert I_2\Vert^2+\Vert I_4\Vert^2 &=& \frac{1}{(1-p)^2 n^2} \sum_{i=1}^n \left\|\nabla f_i(x_i(\lambda))\right\|^2 + \frac{\lambda^2}{p^2 n^2} \sum_{i=1}^n \left\| x_i(\lambda)-\bar{x}(\lambda)\right\|^2 \\
&\leq& \frac{\sigma^2}{p(1-p)}. 
\end{eqnarray*}
Hence 
\begin{eqnarray*}
\mathbb{E}[\Vert G(x)-G(x(\lambda))\Vert^2]\leq(1-p)(1+2p) \Vert I_1-I_2\Vert^2+p(3-2p) \Vert I_3-I_4\Vert^2  +8\sigma^2. 
\end{eqnarray*}
We apply Lemma \ref{lemma_2} and Lemma \ref{lemma_1} to the functions $f$ and $\psi$, and find 
$$ \Vert I_1-I_2\Vert^2 = \frac{1}{(1-p)^2} \left\|\nabla f(x) -\nabla f(x(\lambda))\right\|^2 \leq \frac{2 L}{n (1-p)^2} D_f(x,x(\lambda)) $$
and 
$$ \Vert I_3-I_4\Vert^2 = \frac{\lambda^2}{p^2} \left\|\nabla \psi(x) -\nabla \psi(x(\lambda))\right\|^2 \leq \frac{2 \lambda^2}{n p^2} D_\psi (x,x(\lambda)). $$
It follows that 
$$\mathbb{E}[\Vert G(x)-G(x(\lambda))\Vert^2]\leq \frac{2(1+2p)L}{n(1-p)} D_f(x,x(\lambda))+\frac{2(3-2p)\lambda^2}{np}D_\psi(x,x(\lambda)) +8\sigma^2. 
$$
Since $D_f+\lambda D_\psi = D_F$ and $\nabla F(x(\lambda))=0$, we can continue our analysis as 
\begin{align*}
\mathbb{E}[\Vert G(x)-G(x(\lambda))\Vert^2]&\le\frac{2}{n}
\max\left\{\frac{(1+2p)L}{1-p},\frac{(3-2p)\lambda}{p}\right\}D_F(x,x(\lambda))+8\sigma^2\\
&=\frac{2}{n}
\max\left\{\frac{(1+2p)L}{1-p},\frac{(3-2p)\lambda}{p}\right\}(F(x)-F(x(\lambda)))+8\sigma^2\\
&=2\mathcal{L}(F(x)-F(x(\lambda)))+8\sigma^2.
\end{align*}
This proves the first desired estimate. 

For the second estimate, we express $G(x)= \left\{G(x)-G(x(\lambda))\right\} + G(x(\lambda))$ and bound 
$\Vert G(x)\Vert^2$ by $2 \Vert G(x)-G(x(\lambda))\Vert^2 + 2 \Vert G(x(\lambda))\Vert^2$. Then the desired estimate follows from the expression 
\begin{align*}
\mathbb{E}[\Vert G(x(\lambda))\Vert^2] &= (1-p)\left\|\frac{\nabla f(x(\lambda))}{1-p}\right\|^2+p\left\|\frac{\lambda\nabla \psi(x(\lambda))}{p}\right\|^2 \\
&=\frac{1}{n^2}\sum_{i=1}^{n}\left(\frac{1}{1-p}\Vert\nabla f_i(x_i(\lambda))\Vert^2+\frac{\lambda^2}{p}\Vert x_i(\lambda)-\bar{x}(\lambda)\Vert^2\right) =
\sigma^2.
\end{align*}
The proof of Lemma \ref{lemma_3} is complete. 
\end{proof}

\subsection{Necessary and sufficient condition for the convergence}

L2GD is an SGD method, so we can use classical approaches for SGD algorithms~\cite{ying2008online, yao2010complexity, gower2019sgd} to study the convergence. 

\begin{proof}[Proof of Theorem~\ref{theorem_1}] 
We follow the steps illustrated at the beginning of this section and apply the upper bound (\ref{middletermest}) and lower one (\ref{middletermestlower}) to estimate the residual error $r^k=x^k-x^* =x^k - x(\lambda)$ iteratively. Take $\mu_F =\frac{\mu}{n}$. 

\emph{Sufficiency.} Since $F$ is $\mu_F$-strongly convex according to Lemma \ref{lemma_2}, the upper bound (\ref{middletermest}) asserts 
$$
\mathbb{E}_{x_k}\Vert r^{k+1}\Vert^2 \leq \left(1- \alpha_k \mu_F\right) \Vert r^k\Vert^2 -2\alpha_k \left(F(x^k) - F(x^*)\right)+ \alpha_k^2\mathbb{E}_{x_k}\Vert G(x^k)\Vert^2.
$$
Plugging the second bound in Lemma \ref{lemma_3} with $x=x^k$ into the above estimate, we have 
$$
\mathbb{E}_{x_k}\Vert r^{k+1}\Vert^2 \leq \left(1- \alpha_k \mu_F\right) \Vert r^k\Vert^2 +2\alpha_k (2\mathcal{L} \alpha_k-1) \left(F(x^k) - F(x^*)\right)+ 18 \alpha_k^2 \sigma^2.
$$
Since $x^*$ is a minimizer of $F$, we know $F(x^k) - F(x^*)\geq 0$. 
But $0<\alpha_k\le\frac{1}{2\mathcal{L}}$. Hence the above middle term is nonpositive and we have 
$$
\mathbb{E}_{x_k}[\Vert r^{k+1}\Vert^2] \le (1-\alpha_k\mu_F) \Vert r^{k}\Vert^2 + 18\alpha_k^2\sigma^2.
$$
Applying this bound recursively leads to 
\begin{eqnarray}\label{one_iter_bound}
\mathbb{E}_{x_k, x_{k-1}, \ldots, x_1}[\Vert r^{k+1}\Vert^2] &\le& \prod_{i=1}^{k}(1-\mu_F\alpha_i) \Vert r^{1}\Vert^2 \nonumber\\
&& +\sum_{i=1}^{k-1}\prod_{j=i+1}^k (1-\mu_F\alpha_j)18\sigma^2 \alpha_i^2+ 18\sigma^2 \alpha_k^2.
\end{eqnarray}

Now we show that $\mathbb{E}_{x_k, x_{k-1}, \ldots, x_1}[\Vert r^{k+1}\Vert^2]$ can be arbitrarily small when $k$ is large enough. Let $\varepsilon>0$. 

We separate the right-hand side of (\ref{one_iter_bound}) into four terms. The first term involves 
$$\prod_{i=1}^{k}(1-\mu_F\alpha_i) 
\leq \prod_{i=1}^{k} \exp\left\{-\mu_F\alpha_i\right\}\le\exp\left\{-\mu_F\sum_{i=1}^k\alpha_i\right\}. $$
Since $\sum_{k=1}^\infty\alpha_k=\infty$, there we can find some $k_1 \in \mathbb{N}$ such that for $k\ge k_1$, there holds $\sum_{i=1}^k \alpha_i\ge\frac{1}{\mu_F}\log\frac{1}{\varepsilon}$ which implies 
$$ \prod_{i=1}^{k}(1-\mu_F\alpha_i) \mathbb{E}_{x_1} [\Vert r^{1}\Vert^2] \leq \mathbb{E}_{x_1} [\Vert r^{1}\Vert^2] \varepsilon. $$

Since $\lim_{k\to\infty}\alpha_k=0$, there exists $k_2 \in \mathbb{N}$ such that $\alpha_i\le\mu_F\varepsilon$ whenever $i\ge k_2$. Let $k \ge k_2+2$. The second term on the right-hand side of (\ref{one_iter_bound}) is 
$$18\sigma^2\sum_{i=k_2+1}^{k-1}\prod_{j=i+1}^k (1-\mu_F\alpha_j)\alpha_i^2. $$
Since $\alpha_i\le\mu_F\varepsilon$ for every summand, we know that this term is bounded by 
\begin{align*}
&18\sigma^2 \varepsilon\sum_{i=k_2+1}^{k-1}\prod_{j=i+1}^k (1-\mu_F\alpha_j)\mu_F\alpha_i\\
&=18\sigma^2 \varepsilon\sum_{i=k_2+1}^{k-1}\prod_{j=i+1}^k (1-\mu_F\alpha_j)[1-(1-\mu_F\alpha_i)]\\
&=18\sigma^2 \varepsilon\sum_{i=k_2+1}^{k-1}\left\{\prod_{j=i+1}^k (1-\mu_F\alpha_j)-\prod_{j=i}^k (1-\mu_F\alpha_j)\right\}\\
&=18\sigma^2 \varepsilon \left\{(1-\mu_F\alpha_k)-\prod_{i=k_2 +1}^k (1-\mu_F\alpha_i)\right\}\le 18\sigma^2 \varepsilon.
\end{align*}

The third term on the right-hand side of (\ref{one_iter_bound}) is 
$$18\sigma^2\sum_{i=1}^{k_2}\prod_{j=i+1}^k (1-\mu_F\alpha_j)\alpha_i^2 
 \le \frac{18\sigma^2}{4\mathcal{L}^2}\sum_{i=1}^{k_2}\prod_{j=i+1}^k (1-\mu_F\alpha_j). $$
Since $\sum_{k=k_2 +1}^\infty\alpha_k=\infty$, there exists $k_3 \geq k_2 +1$ such that 
\[\sum_{j=k_2+1}^k \alpha_j \ge\frac{1}{\mu_F}\log\frac{k_2}{\varepsilon}, \qquad \forall k \geq k_3.\]
Then for $k\ge k_3$, we have 
\begin{align*}
18\sigma^2 \sum_{i=1}^{k_2}\prod_{j=i+1}^k (1-\mu_F\alpha_j)\alpha_i^2&\le \frac{9\sigma^2}{2\mathcal{L}^2} \sum_{i=1}^{k_2}\prod_{j=i+1}^k \exp\left\{-\mu_F\alpha_j\right\} \\
&\le \frac{9\sigma^2}{2\mathcal{L}^2} \sum_{i=1}^{k_2}\prod_{j=k_2 +1}^k \exp\left\{-\mu_F\alpha_j\right\} \\
&= \frac{9\sigma^2}{2\mathcal{L}^2}  k_2 \exp\left\{-\mu_F\sum_{j=k_2+1}^k\alpha_j\right\}\\
&\le \frac{9\sigma^2}{2\mathcal{L}^2}  k_2\cdot\frac{\varepsilon}{k_2}=\frac{9\sigma^2}{2\mathcal{L}^2}\varepsilon.
\end{align*}
Thus, when $k\ge\max\{k_1, k_3 +1\}$, since $\alpha_k \leq \frac{1}{2\mathcal{L}}$ and $\alpha_k \leq \mu_F \varepsilon$, we have
$$
\mathbb{E}_{x_k, x_{k-1}, \ldots, x_1}[\Vert r^{k+1}\Vert^2] \le \left\{\mathbb{E}_{x_1} [\Vert r^{1}\Vert^2] + 18\sigma^2 +\frac{9\sigma^2}{2\mathcal{L}^2} + \frac{9\sigma^2}{\mathcal{L}} \mu_F\right\} \varepsilon. 
$$
This verifies $\lim_{k\to\infty}\mathbb{E}_{x_k, x_{k-1}, \ldots, x_1}[\Vert x^k-x(\lambda)\Vert^2]=0$. 

\emph{Necessity.} We apply the lower bound (\ref{middletermestlower}) for the residual error $r^k=x^k-x^*$ and know from the convergence $\lim_{k\to\infty}\mathbb{E}_{x_k, x_{k-1}, \ldots, x_1}[\Vert x^{k+1}-x^*\Vert^2]=0$ that $\lim_{k\to\infty} \alpha_k^2\mathbb{E}_{x_{k-1}, \ldots, x_1}\Vert G(x^k)\Vert^2 =0$. 

Observe that the difference between $\mathbb{E}\Vert G(x^k)\Vert^2 = \mathbb{E}_{x_{k-1}, \ldots, x_1}\Vert G(x^k)\Vert^2$ and $\mathbb{E}\Vert G(x (\lambda))\Vert^2$ can be estimated as 
\begin{eqnarray*}
&&\left|\mathbb{E}\Vert G(x^k)\Vert^2- \mathbb{E}\Vert G(x (\lambda))\Vert^2\right| \\ 
&=&  \Big|\frac{1}{1-p} \mathbb{E}\left\|\nabla f(x^k)\right\|^2 + \frac{\lambda^2}{p} \mathbb{E}\left\|\nabla\psi(x^k)\right\|^2 \\
&&- \frac{1}{1-p} \mathbb{E}\left\|\nabla f(x(\lambda))\right\|^2 - \frac{\lambda^2}{p} \mathbb{E}\left\|\nabla\psi(x(\lambda))\right\|^2\Big| \\
&\leq& \frac{1}{1-p} \mathbb{E}\left\{\left\|\nabla f(x^k)- \nabla f(x(\lambda))\right\| \left( \left\|\nabla f(x^k)\right\| + \left\|\nabla f(x(\lambda))\right\|\right)\right\} \\
&&+ \frac{\lambda^2}{p} \mathbb{E}\left\{\left\|\nabla\psi(x^k)- \nabla\psi(x(\lambda))\right\| \left( \left\|\nabla\psi(x^k)\right\| + \left\|\nabla\psi(x(\lambda))\right\|\right)\right\}.
\end{eqnarray*}
Applying Lemma \ref{lemma_1} to the functions $f, \psi$ and points $x^k, x(\lambda)$, we obtain  
\begin{eqnarray*} 
&&\left\|\nabla f(x^k)- \nabla f(x(\lambda))\right\| \left( \left\|\nabla f(x^k)\right\| + \left\|\nabla f(x(\lambda))\right\|\right) \\
&\leq& L_f \left\|x^k-x(\lambda)\right\| \left(L_f \left\|x^k-x(\lambda)\right\| + 2\left\|\nabla f(x(\lambda))\right\|\right). 
\end{eqnarray*}
and a similar bound for the term with $\psi$. It follows that 
\begin{eqnarray*}
&&\left|\mathbb{E}\Vert G(x^k)\Vert^2- \mathbb{E}\Vert G(x (\lambda))\Vert^2\right| \\ 
&\leq& \frac{1}{1-p} \left\{L_f^2 \mathbb{E}\left\|x^k-x(\lambda)\right\|^2 + 2 L_f\left\|\nabla f(x(\lambda))\right\| \left(\mathbb{E}\left\|x^k-x(\lambda)\right\|^2\right)^{1/2}\right\} \\
&&+ \frac{\lambda^2}{p} \left\{L_\psi^2 \mathbb{E}\left\|x^k-x(\lambda)\right\|^2 + 2 L_\psi\left\|\nabla \psi(x(\lambda))\right\| \left(\mathbb{E}\left\|x^k-x(\lambda)\right\|^2\right)^{1/2}\right\}
\end{eqnarray*}
which converges to $0$ as $k\to\infty$. Since $0<\alpha_k\le\frac{1}{2\mathcal{L}}$ and $\lim_{k\to\infty} \alpha_k^2\mathbb{E}\Vert G(x^k)\Vert^2 =0$, 
we see $\lim_{k\to\infty} \alpha_k^2 \Vert G(x(\lambda))\Vert^2 =0$. But $\Vert G(x(\lambda))\Vert^2= \sigma^2>0$, we must have $\lim_{k\to\infty} \alpha_k =0$. This verifies the first necessary condition which guarantees the existing of some $k_0 \in \mathbb{N}$ such that $\alpha_k\le\frac{1}{4 L_F}$ for $k\ge k_0$.  

To derive the second necessary condition, we use the lower bound $\mathbb{E}_{x_k}\Vert r^{k+1}\Vert^2 \geq \left(1- 2 \alpha_k L_F\right) \Vert r^k\Vert^2$ obtained from (\ref{middletermestlower}) 
 and the elementary inequality $\log(1-x)\ge-2x$ for $0<x\le\frac{1}{2}$ and find 
 \begin{align*}
   \mathbb{E}\Vert r^{k+1}\Vert^2 &\ge\prod_{i=k_0}^{k}(1 - 2 L_F \alpha_i )\prod_{i=1}^{k_0}(1 - 2 L_F \alpha_i ) \Vert r^{1}\Vert^2\\
   &\ge\exp\Big(-4 L_F\sum_{i=k_0}^k \alpha_i\Big)\prod_{i=1}^{k_0}(1 - 2 L_F \alpha_i ) \Vert r^{1}\Vert^2.
 \end{align*}
 Since $\prod_{i=1}^{k_0}(1 - 2 L_F \alpha_i ) \Vert r^{1}\Vert^2 >0$ and $\lim_{k\to\infty}\mathbb{E}\Vert r^{k+1}\Vert^2=0$, we must have $\sum_{k=1}^\infty\alpha_k=\infty$.
This completes the proof of Theorem~\ref{theorem_1}. 
\end{proof}

\subsection{Deriving rates of convergence}

The rates of convergence stated in Theorem~\ref{theorem_2} are derived by applying the upper bound (\ref{one_iter_bound}) obtained in the sufficiency part of the proof of Theorem~\ref{theorem_1} and Lemma \ref{lemma_4}. 

\begin{proof}[Proof of Theorem~\ref{theorem_2}.]
Following the upper bound (\ref{one_iter_bound}) obtained in the proof of Theorem~\ref{theorem_1}, we know from the choice $\alpha_k=\alpha_1 k^{-\theta}\le\frac{1}{2\mathcal{L}}$ that the residual error  $r^{k+1}=x^{k+1}-x(\lambda)$ can be estimated as 
\begin{eqnarray*}
\mathbb{E}[\Vert r^{k+1}\Vert^2] &\le& \prod_{i=1}^{k}\left(1-\frac{\mu_F \alpha_1}{i^{\theta}}\right) \Vert r^{1}\Vert^2 \\
&& +18\sigma^2 \sum_{i=1}^{k-1}\prod_{j=i+1}^k \left(1-\frac{\mu_F \alpha_1}{j^{\theta}}\right) \frac{\alpha_1^2}{i^{2\theta}}+ \frac{18\sigma^2 \alpha_1^2}{k^{2\theta}}.
\end{eqnarray*}

The first term in the above bound is easy to deal with. We have 
\begin{align*}
&\prod_{i=1}^{k}\left(1-\frac{\mu_F \alpha_1}{i^{\theta}}\right) 
\le\exp\left\{-\sum_{i=1}^{k}\frac{\mu_F \alpha_1}{i^{\theta}}\right\}\le\exp\left\{-\mu_F \alpha_1\int_{1}^{k+1} x^{-\theta}dx \right\}. 
\end{align*} 
So, when $0<\theta<1$, we have $\int_{1}^{k+1} x^{-\theta}dx = \frac{(k+1)^{1-\theta}-1}{1-\theta}$ and then apply Lemma~\ref{lemma_4} with $a=\frac{1}{1-\theta}, v= \frac{\mu_F \alpha_1}{1-\theta}$ to obtain 
$$ \prod_{i=1}^{k}\left(1-\frac{\mu_F \alpha_1}{i^{\theta}}\right) 
\le\exp\Bigg\{\frac{\mu_F \alpha_1}{1-\theta}\Bigg\}\left(\frac{1}{\mu_F \alpha_1 e}\right)^\frac{1}{1-\theta}k^{-1}. $$ 
When $\theta=1$, we also have $\int_{1}^{k+1} x^{-\theta}dx = \log (k+1)$ and then 
$$ \prod_{i=1}^{k}\left(1-\frac{\mu_F \alpha_1}{i^{\theta}}\right) 
\le (k+1)^{-\mu_F \alpha_1}. $$ 

The second term without a constant factor can be bounded first as 
\begin{align*}
&\sum_{i=1}^{k-1}\prod_{j=i+1}^k \left(1-\frac{\mu_F \alpha_1}{j^{\theta}}\right) \frac{1}{i^{2\theta}} \leq 
\sum_{i=1}^{k-1} \frac{1}{i^{2\theta}} \exp\Bigg\{-\mu_F \alpha_1\sum_{j=i+1}^{k} j^{-\theta}\Bigg\} = I_1 + I_2:\\
&=\sum_{k/2\le i\le k} \frac{1}{i^{2\theta}} \exp\Bigg\{-\mu_F \alpha_1\sum_{j=i+1}^{k} j^{-\theta}\Bigg\} +\sum_{1\le i< k/2} \frac{1}{i^{2\theta}} \exp\Bigg\{-\mu_F \alpha_1\sum_{j=i+1}^{k} j^{-\theta}\Bigg\}.
\end{align*} 
Then we bound $I_1$ and $I_2$ separately. 

Consider the case $0<
\theta<1$. For $I_1$, we observe that 
$$\sum_{j=i+1}^k j^{-\theta}\ge\int_{i+1}^{k+1}x^{-\theta}dx= \frac{(k+1)^{1-\theta}}{1-\theta}\left(1-\left(\frac{i+1}{k+1}\right)^{1-\theta}\right) $$
and the elementary inequality $(1-x)^{1-\theta}\le 1-(1-\theta)x$ for $0\leq x\leq 1$ yields 
$$\left(\frac{i+1}{k+1}\right)^{1-\theta} = \left(1-\frac{k-i}{k+1}\right)^{1-\theta} \leq 1- (1-\theta) \frac{k-i}{k+1}, $$
we have 
\begin{align*}
I_1&\le\sum_{k/2\le i\le k} \left(\frac{2}{k}\right)^{2\theta} \exp\Bigg\{-{\mu_F \alpha_1(k+1)^{-\theta}}(k-i)\Bigg\}\\
&\le \left(\frac{2}{k}\right)^{2\theta} \left\{1+  \int_0^{\frac{k}{2}}\exp\Bigg\{-{\mu_F \alpha_1 (k+1)^{-\theta}}x\Bigg\}dx\right\} \\
&\le  \left(\frac{2}{k}\right)^{2\theta} \left\{1+ \frac{(k+1)^\theta}{\mu_F \alpha_1}\right\}\le\frac{2^{3\theta}}{\mu_F \alpha_1}k^{-\theta}. 
\end{align*}
For $I_2$, we notice that $\frac{i+1}{k+1} \leq \frac{(k-1)/2 +1}{k+1} \leq \frac{1}{2}$ for $1 \leq i < k/2$ and $1/i^{2\theta} \leq 1$, and thereby 
$$
I_2 \le \frac{k}{2} \exp\Bigg\{-\frac{\mu_F \alpha_1 (1-2^{\theta-1})}{1-\theta}(k+1)^{1-\theta}\Bigg\}$$
which in connection with Lemma~\ref{lemma_4} with $a=\frac{2}{1-\theta}$ and $v=\frac{\mu_F \alpha_1 (1-2^{\theta-1})}{1-\theta}$ yields 
$$ I_2 \leq  \frac{k}{2}\left(\frac{2}{\mu_F \alpha_1 e(1-2^{\theta-1})}\right)^{\frac{2}{1-\theta}}(k+1)^{-2} \le \frac{1}{2}\left(\frac{2}{\mu_F \alpha_1 e(1-2^{\theta-1})}\right)^{\frac{2}{1-\theta}} k^{-1}.
$$
Combining the above estimates, we have 
$$
\mathbb{E}[\Vert r^{k+1}\Vert^2]\le C_{\mu, \alpha_1, \theta, \Vert r^{1}\Vert, n} k^{-\theta}$$
where $C_{\mu, \alpha_1, \theta, \Vert r^{1}\Vert, n}$ is a constant independent of $k$ given by 
\begin{align*} C_{\mu, \alpha_1, \theta, \Vert r^{1}\Vert, n} &= \exp\Bigg\{\frac{\mu_F \alpha_1}{1-\theta}\Bigg\}\left(\frac{1}{\mu_F \alpha_1 e}\right)^\frac{1}{1-\theta}\Vert r^{1}\Vert^2 \\
&+ 18\sigma^2 \alpha_1^2 \left(\frac{2^{3\theta}}{\mu_F \alpha_1} + \frac{1}{2}\left(\frac{2}{\mu_F \alpha_1 e(1-2^{\theta-1})}\right)^{\frac{2}{1-\theta}} +1\right). 
\end{align*}
This verifies the desired rate of convergence (\ref{ratesofconvergence}) in the case $0<\theta <1$. 

For the case $\theta=1$, we also have 
\begin{align*} 
I_1 + I_2 &\leq \sum_{i=1}^{k-1} \frac{1}{i^{2}} \exp\Bigg\{-\mu_F \alpha_1 \int_{i+1}^{k+1}\frac{1}{x}dx \Bigg\} \\
&= (k+1)^{-\mu_F \alpha_1} \sum_{i=1}^{k-1} \frac{(i+1)^{\mu_F \alpha_1}}{i^2} 
\leq 2^{\mu_F \alpha_1} (k+1)^{-\mu_F \alpha_1} \sum_{i=1}^{k-1} i^{\mu_F \alpha_1 -2}.  
\end{align*}
But 
$$ \sum_{i=1}^{k-1} i^{\mu_F \alpha_1 -2} \leq \left\{\begin{array}{ll} 
1 + \int_1^\infty x^{\mu_F \alpha_1 -2} d x \leq 1 + \frac{1}{1-\mu_F \alpha_1}, & \hbox{if}  \ \mu_F \alpha_1 <1, \\
1 + \int_1^k \frac{1}{x} d x \leq 1 + \log k, & \hbox{if}  \ \mu_F \alpha_1 =1, \\
1 + \int_1^{k} x^{\mu_F \alpha_1 -2} d x \leq 1 + \frac{k^{\mu_F \alpha_1 -1} - 1}{\mu_F \alpha_1 -1}, & \hbox{if}  \ \mu_F \alpha_1 > 1, 
\end{array}\right. $$
Therefore, we have 
$$
\mathbb{E}[\Vert r^{k+1}\Vert^2]\le C_{\mu, \alpha_1, \theta, \Vert r^{1}\Vert, n} 
\left\{\begin{array}{ll} 
k^{-\mu_F \alpha_1}, & \hbox{if}  \ \mu_F \alpha_1 <1, \\
\frac{1 + \log k}{k}, & \hbox{if}  \ \mu_F \alpha_1 =1, \\
\frac{1}{k}, & \hbox{if}  \ \mu_F \alpha_1 > 1. 
\end{array}\right. $$
where the constant $C_{\mu, \alpha_1, \theta, \Vert r^{1}\Vert, n}$ is independent of $k$ given by 
\begin{align*} C_{\mu, \alpha_1, \theta, \Vert r^{1}\Vert, n} &=  \Vert r^{1}\Vert^2 + 18\sigma^2 \alpha_1^2 
+ 18\sigma^2 \alpha_1^2 2^{\mu_F \alpha_1} \left\{\begin{array}{ll} 
\frac{2- \mu_F \alpha_1}{1-\mu_F \alpha_1}, & \hbox{if}  \ \mu_F \alpha_1 <1, \\
1, & \hbox{if}  \ \mu_F \alpha_1 =1, \\
\frac{\mu_F \alpha_1}{\mu_F \alpha_1 -1}, & \hbox{if}  \ \mu_F \alpha_1 > 1. 
\end{array}\right.
\end{align*}
This verifies the desired rate of convergence (\ref{ratesofconvergence}) for the case $\theta=1$ and completes the proof of Theorem~\ref{theorem_2} by taking $\mu_F =\frac{\mu}{n}$. 
\end{proof}


\end{document}